\documentclass[journal]{IEEEtran}

\RequirePackage{amsthm,amsmath,amssymb}
\RequirePackage[numbers]{natbib} 
\usepackage{algorithm,algorithmic,enumitem,color,xspace,url,mathtools}
\usepackage{graphicx}

\newtheorem{thm}{Theorem}
\newtheorem{lem}{Lemma}
\newtheorem{cor}{Corollary}
\newtheorem{prop}{Proposition}
\newtheorem{defn}{Definition}
\newtheorem*{rem*}{Remark}

\global\long\def\mca{\mathcal{A}}

\global\long\def\H{\mathcal{H}}
\global\long\def\D{D}
\global\long\def\d{d}

\global\long\def\gpp{GSBM\xspace}

\global\long\def\ca{c_{\mca}}
\global\long\def\cac{c_{\mca^c}}
\newcommand{\Omegat}[1]{\ensuremath{\tilde{\Omega}\left(#1\right)}}
\DeclareMathOperator{\polylog}{polylog}

\begin{document}

\title{Improved Graph Clustering}

\author{Yudong Chen,
Sujay Sanghavi, and
Huan Xu
\thanks{The work of Y. Chen was supported by NSF grant EECS-1056028 and DTRA grant HDTRA 1-08-0029. The work of S. Sanghavi was supported by the DTRA Young investigator award and NSF grants 1302435, 1017525 and 0954059.  The work of H. Xu was partially supported by the Ministry of Education of Singapore through AcRF Tier Two grant R-265-000-443-112. The material in this paper was presented in part under the title ``Clustering Sparse Graphs'' at the Neural Information Processing Systems Conference, Lake Tahoe, Nevada, United States, 2012.}

\thanks{Y. Chen is with the Department of Electrical Engineering and Computer Sciences at the University of California at Berkeley, Berkeley, CA 94720, USA. S. Sanghavi is with the Department of Electrical and Computer Engineering, The University of Texas at Austin, Austin, TX 78712, USA. H. Xu is with the Department of Mechanical Engineering, National University of Singapore, 9 Engineering Drive 1, Singapore 117575, Singapore. (e-mail: yudong.chen@eecs.berkeley.edu, sanghavi@mail.utexas.edu, mpexuh@nus.edu.sg)}
}

\markboth{}%
{Chen \MakeLowercase{\textit{et al.}}: Improved Graph Clustering}

\maketitle

\begin{abstract}
Graph clustering involves the task of dividing nodes into clusters, so that the edge density is higher within clusters as opposed to across clusters. A natural, classic and popular statistical setting for evaluating solutions to this problem is the stochastic block model, also referred to as the planted partition model.

In this paper we present a new algorithm---a convexified version of Maximum Likelihood---for graph clustering. We show that, in the classic stochastic block model setting, it outperforms existing methods by polynomial factors when the cluster size is allowed to have general scalings. In fact, it is within logarithmic factors of known lower bounds for spectral methods, and there is evidence suggesting that no polynomial time algorithm would do significantly better.

We then show that this guarantee carries over to a more general extension of the stochastic block model. Our method can handle the settings of semi-random graphs, heterogeneous degree distributions, unequal cluster sizes, unaffiliated nodes, partially observed graphs and planted clique/coloring etc. In particular, our results provide the best exact recovery guarantees to date for the planted partition, planted $ k $-disjoint-cliques and planted noisy coloring models with general cluster sizes; in other settings, we match the best existing results up to logarithmic factors.
\end{abstract}

\section{Introduction}

This paper proposes a new algorithm for the following task: given an undirected unweighted graph, assign the nodes into disjoint clusters so that the density of edges within clusters is higher than the edge density across clusters. Clustering arises in applications such as a community detection, user profiling, link prediction, collaborative filtering etc. In these applications, one is often given as input a set of similarity relationships (either ``1" or ``0") and the goal is to identify groups of similar objects. For example, given the friendship relations on Facebook, one would like to detect tightly connected communities, which is useful for subsequent tasks like customized recommendation and advertisement.

Graphs in modern applications have several characteristics that complicate graph clustering:
\begin{itemize}
\item {\bf Small density gap:} the edge density across clusters is only a small additive or multiplicative factor different from within clusters;
\item {\bf Sparsity:} the graph is overall very sparse even within clusters;
\item {\bf High dimensionality:} the number of clusters may grow unbounded as a function of the number of nodes $ n $, which means the sizes of the clusters can be vanishingly small compared to $ n $;
\item {\bf Unaffiliated nodes:} there may exist a large number of nodes that do not belong to any clusters and are loosely connected to the rest of the graph;
\item {\bf Heterogeneity:} the cluster sizes, node degrees and edge densities may be non-uniform across the graph; edge connections may not be well-modeled by a probabilistic distribution, and there may exist hierarchical cluster structures. 
\end{itemize}

Various large modern datasets and graphs have such characteristics~\citep{fortunato2010community,leskovec2008properties}; examples include the web graph and social graphs of various social networks etc. 
As has been well-recognized, these characteristics make clustering more difficult. 
When the in-cluster and across-cluster edge densities are close or there are many unstructured unaffiliated nodes, the clustering structure is less significant and thus harder to detect. Sparsity further reduces the amount of information and makes the problem noisier. In the high dimensional regime, there are many small clusters, which are easy to lose in the noise. Heterogeneous and non-random structures in the graphs foil many algorithms that otherwise perform well; for example, conventional spectral clustering methods are often known to be not robust to heterogeneity in the graphs~\cite{chaudhuri2012extended,coja2004coloringSemirandom}. Finally, the existence of hierarchical structures and unaffiliated nodes renders many existing algorithms and theoretical results inapplicable, as they fix the number of clusters a priori and force each node to be assigned to a cluster. It is desirable to design an algorithm that can handle all these issues in a principled manner.

\subsection{Our Contributions}

Our algorithmic contribution is a new method for unweighted graph clustering. It is motivated by the maximum-likelihood estimator for the  classical Stochastic Blockmodel~\cite{condon2001algorithms} (also known as the {Planted Partition Model}~\cite{holland83}) for random clustered graphs. In particular, we show that this maximum-likelihood estimator can be written as a linear objective over combinatorial constraints; our algorithm is a convex relaxation of these constraints, yielding a convex program overall. While this is the motivation, it performs well---both in theory and empirically---in settings that are not just the standard stochastic blockmodel.

Our main analytical result in this paper is theoretical guarantees on our algorithm's performance; we study it in a {\em semi-random generalized stochastic blockmodel}. This model generalizes not only the standard stochastic blockmodel and planted partition model, but many other classical planted models including planted $ k $-disjoint-cliques~\cite{ames2010kclique,alon1998hiddenClique}, planted dense subgraph~\cite{mcsherry2001spectralpartitioning}, planted coloring~\cite{alon1997coloring3,coja2004coloringSemirandom} and their semi-random variants~\cite{bollobas2004maxcut,feige2001semirandom,krivelevich2006coloring}. Our main result gives the conditions (as a function of the in/cross-cluster edge densities $ p $ and $ q $, the density gap $ |p-q| $, the minimum cluster size $ K $ and the total number of nodes $ n $) under which our algorithm is guaranteed to recover the ground-truth clustering. When $ p>q $, the key condition reads
\begin{equation}\label{eq:simple_cond}
p-q = \Omega\left(\frac{\sqrt{p(1-q)n}}{K} \right);
\end{equation} 
here all the parameters are allowed to scale with $n$. Note that the condition does not depend explicitly on the number of unaffiliated nodes or the number of clusters. An analogous result holds for $ p<q $.

While the planted and stochastic block models have a rich literature, this single result shows that the performance of our algorithm  matches all existing methods (up to at most logarithmic factors) in exact recovery; moreover, in the cases of the standard planted partition/$k$-disjoint-cliques/noisy-coloring models with general scaling of $ p,q $ and $ K $, we achieve order-wise improvement over existing methods, in the sense that our algorithm succeeds for a much larger range of the parameters. In fact, there is evidence indicating that we are close to the boundary at which any polynomial-time algorithm can be expected to work. The proof for our main theorem is relatively simple, relying only on standard concentration results. Our simulation study supports our theoretic finding, that the proposed method is effective in clustering noisy graphs and outperforms existing methods.

The rest of the paper is organized as follows: Section~\ref{sec:related} provides an overview of related work;  Section~\ref{sec:algo} presents our algorithm; Section~\ref{sec:model} describes the Semi-Random Generalized Stochastic Blockmodel, which is a generalization of the standard stochastic blockmodel, one that allows the modeling of the issues mentioned above; Section~\ref{sec:perf} presents the main results---a performance analysis of our algorithm for the semi-random generalized stochastic blockmodel, and provides a detailed comparison to the existing literature and a discussion of the implications for different special cases; Section~\ref{sec:expts} provides simulation results; the proofs of our theoretic results are given in Sections~\ref{sec:proof_mono} to~\ref{sec:proof_est_pq}; the paper concludes with a discussion in Section~\ref{sec:conclusion}

\subsection{Related Work} \label{sec:related}

The general field of clustering, or even graph clustering, is too vast for a detailed survey here; we focus on the most related threads, and therein too primarily on work which provides analytical guarantees on the resulting algorithms.

\subsubsection{Stochastic block models} 

Also called ``planted models"~\cite{condon2001algorithms,holland83}, these are arguably the most natural random clustered graph models. In the simplest or standard setting, $ n $ nodes are partitioned into disjoint subsets of equal size $ K $ (called the true clusters), and then edges are generated independently and at random, with the probability $p$ of an edge between two nodes in the same cluster higher than the probability $q$ for two nodes in different clusters. The task is to recover the true clusters given the graph. 
The parameters $p,q, K$ and $ n $ typically govern whether an algorithm succeeds in recovery or not. 

\begin{table*}
\caption{Comparison with Literature for the Standard Stochastic Blockmodel
\label{tab:comparison}}
\begin{minipage}{\textwidth}
\begin{center}
\begin{tabular}{rlll}
\hline
\footnotetext{To facilitate a direct comparison, this table specializes some of the results to the case where every underlying cluster is of the same size $K$, and the in/cross-cluster edge probabilities are uniformly $p$ and $q$, with $ p>q $. Some of the algorithms above need this assumption, and some---like ours---do not. In the table we use the soft $ \Omegat{\cdot} $ notation, which ignores log factors.}
Paper & Cluster size $K$ & Density gap $p-q$ & Sparsity $ p $\\
\hline
Boppana (1987) \cite{boppana1987eigenvalues} & $n/2$ & $\tilde{\Omega}\left(\sqrt{\frac{p}{n}}\right)$ & $ \Omegat{\frac{1}{n}} $\\
Jerrum et al. (1998) \cite{jerrum1998metropolis} & $n/2$ & $\tilde{\Omega}\left(\frac{1}{n^{1/6-\epsilon}}\right)$ & $\Omegat{n^{1/6-\epsilon}} $\\
Condon et al. (2001) \cite{condon2001algorithms} & $\Theta\left(n\right)$ & $\tilde{\Omega}\left(\frac{1}{n^{1/2-\epsilon}}\right)$ & $\Omegat{n^{1/2-\epsilon}} $\\
Carson et al. (2001) \cite{carson2001planted} & $n/2$ & $\tilde{\Omega}\left(\sqrt{\frac{p}{n}}\right)$ & $\Omegat{\frac{1}{n}} $\\
Feige et al. (2001) \cite{feige2001semirandom} & $n/2$ & $\tilde{\Omega}\left(\sqrt{\frac{p}{n}}\right)$ & $\Omegat{\frac{1}{n}} $\\
McSherry (2001) \cite{mcsherry2001spectralpartitioning} & ${\Omega}\left(n^{2/3}\right)$ & $\tilde{\Omega}\left(\sqrt{\frac{pn^2}{K^3}}\right)$ & $\Omegat{\frac{n^2}{K^3}}  $\\
Bollobas (2004) \cite{bollobas2004maxcut}  & $\Theta\left(n\right)$  & $\tilde{\Omega}\left(\sqrt{\frac{q }{n}} \vee \frac{1}{n}\right)$  & $\Omegat{\frac{1}{n}} $\\
Giesen et al. (2005) \cite{giesen2005manypartition} & ${\Omega}\left(\sqrt{n}\right)$ & $\tilde{\Omega}\left(\frac{\sqrt{n}}{K}\right)$ & $\Omegat{\frac{\sqrt{n}}{K}}$\\
Shamir (2007) \cite{shamir2007improved} & ${\Omega}\left(\sqrt{n}\log n\right)$ & $\tilde{\Omega}\left(\frac{\sqrt{n}}{K}\right)$ & $\Omegat{\frac{\sqrt{n}}{K}}$\\
Coja-Oghlan (2010) \cite{coja2010adaptive} & $\Omega(n^{4/5})$ & $\Omegat{\sqrt{\frac{pn^4}{K^5}}}$ & $\Omegat{\frac{n^4}{K^5}}$\\
Rohe et al. (2011) \cite{rohe2011stochasticblock} &  ${\Omega}\left((n\log n)^{2/3}\right)$ & $\tilde{\Omega}\left(\frac{ n^{1/2}}{K^{3/4}}\right)$  & $ \Omegat{\frac{1}{\sqrt{\log n}}} $\\ 
Oymak et al. (2011) \cite{oymak2011finding}& ${\Omega}\left(\sqrt{n}\right)$ & $\tilde{\Omega}\left(\frac{\sqrt{n}}{K}\right)$ & $\Omegat{\frac{\sqrt{n}}{K}}$ \\
Chaudhuri et al. (2012) \cite{chaudhuri2012extended} &${\Omega}\left(\sqrt{n\log n}\right)$  & $\tilde{\Omega}\left(\frac{\sqrt{n}}{K}\right)$ & $\Omegat{\frac{\sqrt{n}}{K}}$\\
Ames (2012) \cite{ames2012clustering} &${\Omega}\left(\sqrt{n}\right)$ & $\tilde{\Omega}\left(\frac{\sqrt{n}}{K}\right)$ & $\Omegat{\frac{\sqrt{n}}{K}}$\\
 &  &   &\\
Our result & ${\Omega}\left(\sqrt{n}\right)$ & $\tilde{\Omega}\left(\frac{\sqrt{pn}}{K}\right)$ & $\Omegat{\frac{n}{K^2}}$\\
\hline 
\end{tabular}
\end{center}
\end{minipage}
\label{table:blockmodel}
\end{table*}

There is now a long line of analytical work on stochastic block models; we focus on methods that allow for \emph{exact recovery} (i.e., every node is correctly classified), and summarize the conditions required by known methods in Table~\ref{table:blockmodel}. As can be seen, we improve over existing methods by polynomial factors for general values of $ K $---in particular, when the cluster size satisfies $ K=n^{1-\alpha} $ for any constant $ \alpha>0 $ (which means the number of clusters is growing at the rate $ n/K = n^{\alpha} $.)\footnote{Our comparison focuses on polynomial factors and the setting with general values of $ K $. We note that in the special case of $ K=\Theta(n) $, some existing results (e.g.,\cite{coja2010adaptive}) are better then ours by logarithmic factors.} In addition, as opposed to several of these methods, our method can handle unaffiliated nodes, heterogeneity, hierarchy in clustering etc, and apply to other models including planted clique and planted noisy coloring.

We would like to mention two recent results that appeared after the conference version~\cite{chen2012sparseNIPS} of this paper. The work in~\cite{anandkumar2013tensorMixedCommunity} shows that a computationally efficient tensor decomposition approach succeeds for the standard stochastic blockmodel when $ p-q=\Omega\left({\sqrt{pn}\polylog n}/K\right) $; our guarantee~\eqref{eq:simple_cond} is better by a factor of $ \Theta(\polylog n/\sqrt{1-q})$. Moreover, for the standard planted clique model ($ p=1,q=1/2 $), we only require the clique size to be $ K=\Omega(\sqrt{n}) $, better than their requirement $ K =\tilde{\Omega}(n^{2/3}) $.  Another subsequent work~\cite{ailon2013breaking} considers the setting with  heterogeneous cluster sizes and no unaffiliated nodes, and shows that our algorithm can be combined with an iterative reduction procedure to sequentially recover clusters smaller than is allowed in this paper.

A complimentary line of work has investigated {\em lower bounds} for the stochastic blockmodel; i.e., for what values/scalings of $p,q$ and $K$ it is \emph{not} possible (either for any algorithm, or for any polynomial-time algorithm) to recover the underlying true clusters~\cite{chaudhuri2012extended,decelle2011asymptotic,nadakuditi2012spectra}. We discuss and compare with these two lines of work in more details in the main results section.

\subsubsection{Convex methods for matrix decomposition} 

Our method is related to recent literature on the recovery of low-rank matrices using convex optimization, and in particular the recovery of such matrices from ``sparse" perturbations (i.e., where a fraction of the elements of the low-rank matrix are possibly arbitrarily modified, while others are untouched). Sparse and low-rank matrix decomposition using convex optimization was initiated by~\cite{chandrasekaran2011siam,candes2009robustPCA};  follow-up works~\cite{chen2011isit,li2013constantCorruption} have the current state-of-the-art guarantees on this problem, and~\cite{Jalali2011clustering} applies it directly to graph clustering. 

The method in this paper is Maximum Likelihood, but it can also be viewed as a weighted version of sparse and low-rank matrix decomposition, with {\em different elements of the sparse part penalized differently, based on the given input graph.} There is currently little work or analysis on weighted matrix decomposition; in that sense, while our weights have a natural motivation in our setting, our recovery results are likely to have broader implications, for example robust versions of PCA when not all errors are created equal but have a corresponding prior.

\section{Algorithm} \label{sec:algo}

We now describe our algorithm. As mentioned, it is a convex relaxation of the Maximum Likelihood (ML) estimator as applied to the standard stochastic blockmodel. So, in what follows, we first develop notation and the exact ML estimator, and then its relaxation.

{\bf ML for the standard stochastic blockmodel:} Recall that in the standard stochastic blockmodel nodes are divided into disjoint clusters, and edges in the graph are chosen independently; the probability of an edge between a pair of nodes in the same cluster is $p$, and for a pair of nodes in different clusters it is $q$. Given the graph, the task is to find the underlying clusters that generated it. To write down the ML estimator for this, let us represent any candidate clustering by a corresponding \emph{cluster matrix} $ Y\in \mathbb{R}^{n\times n} $  where $y_{ij} =1 $ if and only if nodes $i $ and~$ j $ are assigned to the same cluster,\footnote{Throughout this paper, for any matrix $ M $, $ m_{ij} $ denotes its $ (i,j) $-th entry.} and $ y_{ij}=0 $ otherwise; in particular, $ y_{ii} = 1 $ for any node $ i $ that belongs to a cluster. Any $Y$ thus needs to have a block-diagonal structure, with each block being all $ 1 $'s. 

A vanilla ML estimator then involves optimizing a likelihood subject to the combinatorial constraint that the search space is the cluster matrices. Let $ A\in \mathbb{R}^{n\times n} $ be the observed adjacency matrix of the graph (we assume $ a_{ii} = 1 $ for all $ i $); then, the log likelihood function of $ A $ given $ Y $ is
\begin{align*}
 &\log \mathbb{P}(A|Y) \\
=& \log  \left(\quad\;\prod_{\mathclap{(i,j):y_{ij}=1}} p^{a_{ij}}(1-p)^{1-a_{ij}} \quad\prod_{\mathclap{(i,j):y_{ij}=0}} q^{a_{ij}}(1-q)^{1-a_{ij}}\right).  
\end{align*}
We notice that this can be written, via a re-arrangement of terms, as 
\begin{align}
\log \mathbb{P}(A|Y) &= \log\left(\frac{p}{q}\right) \sum_{\mathclap{a_{ij}=1}} y_{ij} - \log\left(\frac{1-q}{1-p}\right) \sum_{\mathclap{a_{ij}=0}} y_{ij}  +C, \label{eq:ml_obj}
\end{align}
where $ C $ collects the terms that are independent of $ Y $. The ML estimator would be maximizing the above expression subject to $ Y $ being a cluster matrix. While the objective is a linear function of $Y$, this optimization problem is combinatorial due to the requirement that $Y$ be a cluster matrix (i.e., block-diagonal with each block being all-ones), and is intractable in general.

{\bf  Our algorithm:} We obtain a convex and tractable algorithm by  replacing the constraint ``$Y$ is a cluster matrix" with (i) the constraints $0\leq y_{ij}\leq 1$ for all pairs $(i,j)$, and (ii) a nuclear norm\footnote{The nuclear norm of a matrix is the sum of its singular values.} regularizer $\|Y\|_*$ in the objective. The latter encourages $Y$ to be {\em low-rank}, and is based on the well-established insight that a cluster matrix  has low rank---in particular, its rank equals the number of clusters. (We discuss other related relaxations later in this section.)

Also notice that the likelihood expression~\eqref{eq:ml_obj} is linear in $ Y $ and only the \emph{ratio} of the two coefficients $ \log(p/q) $ and $ \log((1-q)/(1-p)) $ is important. We therefore introduce a parameter $ t $ which allows us to choose any ratio. This has the advantage that instead of knowing both $ p $ and $ q $, we only need to choose one number $ t $ (which should be between $ p $ and $ q $; we remark on how to choose $ t $ later). This leads to the following convex formulation: 
\begin{align}
\max_{Y\in\mathbb{R}^{n\times n}} & \quad \ca\sum_{ a_{ij}=1}y_{ij} \; -  \cac\sum_{a_{ij}=0}y_{ij} \; - 48\sqrt{n}\left\Vert Y\right\Vert _{*} \label{eq:cvx_prog}\\
\mbox{s.t.} & \quad 0\le y_{ij}\le1,\forall i,j.\label{eq:inequalities}
\end{align}
where the weights $ c_{\mca} $ and $ c_{\mca^c} $ are given by
\begin{equation}
c_{\mca}=\sqrt{\frac{1-t}{t}} \quad \mbox{ and } \quad c_{\mca^{c}}=\sqrt{\frac{t}{1-t}}.
\label{eq:weights}
\end{equation}  
Here the factor ${48\sqrt{n}} $ balances the contributions of the nuclear norm and the likelihood; the specific values of this factor as well as of $ \ca $ and $ \cac $ are derived from our analysis (cf. Section~\ref{sec:proof_WisGood}). The optimization problem~\eqref{eq:cvx_prog}--\eqref{eq:inequalities} is convex and can be cast as a Semidefinite Program (SDP)~\cite{chandrasekaran2011siam,recht2010guaranteed}. More importantly, it can be solved using efficient first-order methods for large graphs (see  Section~\ref{sec:implementation}).

Our algorithm is given as Algorithm~\ref{alg:clustering}. Depending on the given $A$ and the choice of $ t $, the optimal solution $\widehat{Y}$ may or may not be a cluster matrix.  Checking if a given $ \widehat{Y} $ is a cluster matrix can be done easily, e.g., via an SVD, which will also reveal the cluster memberships if it is a cluster matrix. If it is not, any one of several rounding/aggregation ideas (e.g., the one in~\cite{mathieu2010correlation}) can be used empirically; we do not delve into this approach in this paper, and simply output failure. In Section~\ref{sec:perf} we provide sufficient conditions under which $\widehat{Y} $ is guaranteed to be a cluster matrix, with \emph{no} rounding required. 
\begin{algorithm}
\caption{Convex Clustering}
\label{alg:clustering} 
\begin{algorithmic}
	\STATE Input: $ A \in \mathbb{R}^{n \times n}$, $ t \in (0,1) $
	\STATE Solve program~\eqref{eq:cvx_prog}--\eqref{eq:inequalities} with weights~\eqref{eq:weights}. Let $ \widehat{Y} $ be an optimal solution.
	\IF{$ \widehat{Y} $ is a cluster matrix}
		\STATE Output cluster memberships obtained from $ \widehat{Y} $.
	\ELSE
		\STATE Output ``Failure''.
	\ENDIF
\end{algorithmic}
\end{algorithm}

\subsection{Remarks about the Algorithm}\label{sec:remark_alg}

Note that while we derive our algorithm from the standard stochastic blockmodel, our analytical results hold in a much more general setting. In practice, one could execute the algorithm (with appropriate choice of $t$, and hence $ \ca $ and $ \cac $) on any given graph.

\subsubsection*{Tighter relaxations} The formulation~\eqref{eq:cvx_prog}--\eqref{eq:inequalities} is not the only way to relax the non-convex ML estimator. Instead of the nuclear norm regularizer,  a hard constraint $ \Vert Y \Vert_* \le n $ may be used. One may further replace this constraint with the positive semidefinite constraint $ Y\succeq 0$ and the linear constraints $ y_{ii} = 1 $, both satisfied by any cluster matrix.\footnote{The constraints $ y_{ii} =1,\forall i$ are satisfied when there is no unaffiliated node.} It is not hard to check that these modifications lead to convex relaxations with smaller feasible sets, so any performance guarantee for our formulation~\eqref{eq:cvx_prog}--\eqref{eq:inequalities} also applies to these alternative formulations. We choose to focus on our original formulation based on the following theoretical and practical considerations: a) Its performance guarantees apply to the other tighter relaxations as well. 
b) We do not obtain order-wise better theoretical guarantees with these alternative formulations. The work~\cite{mathieu2010correlation} considers these tighter relaxations but does not obtain better exact recovery guarantees than ours. In fact, as we argue in the next section, our guarantees are likely to be order-wise optimal and thus any alternative convex formulations are unlikely to provide significant improvements in a scaling sense. c) Our simpler formulation facilitates efficient solution for large graphs via first-order methods; we describe one such method in Section~\ref{sec:implementation}. 

\subsubsection*{Choice of $ t $} Our algorithm requires an extraneous input~$ t $. For the standard planted $ k$-disjoint-cliques problem~\cite{ames2010kclique,mcsherry2001spectralpartitioning} (with $ k $ disjoint cliques planted in a random graph $G_{n,1/2} $), one can use $ t=3/4 $ (see Section~\ref{sec:clique}). For the standard stochastic blockmodel (with nodes partitioned into equal-size clusters and edge probabilities being uniformly $ p $ and $ q $ inside and across clusters), the value of $ t $ can be determined from the data (see Section~\ref{sec:est_p}). In these cases, our algorithm has no tuning parameters whatsoever and does not require knowledge of the number or sizes of the clusters. For the general setting, $ t $ should be chosen to lie between $ p $ and $ q $, which now represent the lower/upper bounds for the in/cross-cluster edge densities. As such, $ t $ can be interpreted as the \emph{resolution} of the clustering algorithm. To see this, suppose the clusters have a hierarchical structure, where each big cluster is partitioned into smaller sub-clusters with higher edge densities inside. In this case, either level of the clusters, the  top-level big ones or the bottom-level small ones, can be considered as the ground truth, and it is a priori not clear
which of them should be recovered. This ambiguity is resolved by specifying $ t $: our algorithm recovers those clusters with in-cluster edge density higher than $ t $ and cross-cluster density lower than $ t $. With a larger $ t $, the algorithm operates at a higher resolution and detects small clusters with high density. By varying $ t $, our algorithm can be turned into a method for \emph{multi-resolution clustering}~\cite{fortunato2010community} which explores all levels of the cluster hierarchy. We leave this to future work. Importantly, the above example shows that it is generally impossible to uniquely determine the value of $ t $ from data.

\section{The Generalized Stochastic Blockmodel}\label{sec:model}

While our algorithm above is derived as a relaxation of ML estimator for the standard stochastic blockmodel, we establish performance guarantees in a much more general setting. The model is described below, which is defined by six parameters $ n $, $ n_1 $, $ r $, $ K $, $ p $ and $ q $. 

\begin{defn}[Generalized Stochastic Blockmodel (\gpp)]\label{def:GSBM}
The $ n=n_1 + n_2 $ nodes are divided into two sets $ V_1 $ and $ V_2 $.  The $n_1$  nodes in $ V_1 $ are further partitioned into $r$ disjoint sets, which we will refer to as the ``true" clusters. Let $K$ be the {\em minimum size} of a true cluster. If $ p>q $, consider a random graph generated as follows: For every pair of nodes $i,j$ that belong to the same true cluster, edge $(i,j)$ is present in the graph with probability that is {\em at least} $p$, while for every pair where the nodes are in different clusters the edge is present with probability {\em at most} $q$. The other $ n_2 $ nodes in $ V_2 $ are not in any cluster (we will call them unaffiliated nodes); for each $ i\in V_2 $ and $ j\in V_1 \cup V_2 $, there is an edge between the pair $ i,j $ with probability {\em at most} $ q $. If $ p<q $, then the graph is generated similarly as above, except that the probability of an in-cluster edge is \emph{at most} $ p $, while the probability of other edges is \emph{at least} $ q $. Note that it is implicit that $ r\ge 1 $, $ K\ge 1 $ and $ n\ge n_1 \ge rK $. 
\end{defn}

\begin{defn}[Semi-random \gpp]\label{def:SR_GSBM}
On a graph generated from \gpp with $ p>q $ ($ p<q $, resp.), an adversary is allowed to arbitrarily (a) add (remove, resp.) edges between pairs of nodes in the same true cluster, and (b) remove (add, resp.) edges between pairs of nodes if they are in different clusters, or if at least one of them is an unaffiliated node in $ V_2 $.
\end{defn}

The {\bf objective} is to find the underlying true clusters, given the graph generated from the semi-random \gpp. 

The standard stochastic blockmodel/planted partition model is a special case of \gpp with $ n_2=0, r\ge 2 $, all cluster sizes equal to $ K $, and all in-cluster and cross-cluster probabilities equal to $ p $ and $ q $, respectively. \gpp generalizes the standard models as it allows for heterogeneity in the graph:
\begin{itemize}
\item $ p$ and $q$ are lower and upper bounds instead of exact edge probabilities, so nodes can have different degrees; there may also exist nested clusters (cf. Section~\ref{sec:remark_alg}).
\item $ K $ is also a lower bound, so clusters can have different sizes.
\item Unaffiliated nodes (nodes not in any cluster) are allowed.
\end{itemize}
\gpp removes many restrictions in the standard planted models and better models practical graphs.

The semi-random \gpp allows for further modeling power. It blends the worst case models, which are often overly pessimistic,\footnote{For example, the minimum graph bisection problem is NP-hard.} and the purely random graphs, which are extremely unstructured and have very special properties usually not possessed by real-world graphs~\cite{frieze1997algoRandomGraphs}. This semi-random framework has been used and studied extensively in the computer science literature as a better model for real-world networks~\cite{bollobas2004maxcut,feige2001semirandom,krivelevich2006coloring}, as it allows for \emph{non-randomness} in a graph. Note that the term ``adversary'' means \emph{arbitrary} deviation from the random model (as long as it is allowed by the semi-random model), and it covers, but is not limited to, adversarial deviation. At first glance, the adversary seems to make the problem easier as it adds in-cluster edges and removes cross-cluster edges (when $ p>q $). This is not necessarily  the case. The adversary can significantly change some statistical properties of the random graph (e.g., alter spectral structure and node degrees, and create local optima by adding dense spots~\cite{feige2001semirandom}), and foil algorithms that over-exploit such properties. For example, some spectral algorithms that work well on random models are proved to fail in the semi-random setting~\cite{coja2004coloringSemirandom}. An algorithm that works well in the semi-random setting is likely to be more robust to model mis-specification in real-world applications~\cite{feige2001semirandom}. As shown later, our algorithm processes this desired property. 

\subsection{Special Cases}

\gpp recovers as special cases many classical and widely studied models for clustered graphs, by considering different values for the parameters $ n_1 $, $ n_2 $, $ r $, $ K $, $ p $ and $ q $. We classify these models into two categories based on the relation between $ p $ and $ q $.

\begin{enumerate}
\item \textbf{$ p>q $:} \gpp with $ p>q $ models \emph{homophily}, the tendency that individuals belonging to the same community tend to connect \emph{more} than those in different communities. Special cases include:
\begin{itemize}
    \item \textbf{Planted Clique}~\cite{alon1998hiddenClique}: $ p=1 $, $ r=1 $ (so $ n_1=K $) and $ n_2>0 $; 
    \item \textbf{Planted $ r $-Disjoint-Cliques}~\cite{mcsherry2001spectralpartitioning,ames2010kclique}: $ p=1 $ and $ r\ge 1 $;
    \item \textbf{Planted Dense Subgraph}~\cite{mcsherry2001spectralpartitioning}: $ p< 1 $, $r=1 $ and $ n_2>0 $;
    \item \textbf{Stochastic Blockmodel, Planted Partition}~\cite{holland83,condon2001algorithms}: $ n_2 =0 $, $ r\ge 2 $ with all cluster sizes equal to $ K $. The special case with $ r=2 $ can be call the {Planted Bisection Model}~\cite{condon2001algorithms,bollobas2004maxcut}.
\end{itemize}
\item \textbf{$ p<q $:} 
This is complementary to the homophily case above.
Special cases include:
\begin{itemize}
    \item \textbf{Planted Coloring}~\cite{feige2001semirandom}: $ q>p=0 $, $ r\ge 2 $, and $ n_2 = 0 $;
    \item \textbf{Planted $ r $-Cut, Planted Noisy Coloring}~\cite{bollobas2004maxcut,decelle2011asymptotic}: $ q>p\ge0 $, $ r\ge 2 $, and $ n_2 = 0 $.
\end{itemize} 
\end{enumerate}
Recall that the max-clique, max-cut, graph partition and graph coloring problems are all NP-hard in the worst case~\cite{garey,alon1998hiddenClique,alon1997coloring3,condon2001algorithms}. The above ``planted'' variants of these problems are standard models for studying their average-case behavior.

In the next section, we provide performance guarantees for our algorithm under the semi-random \gpp. This implies guarantees for all the special cases above. We provide a detailed comparison with literature after presenting our results.

\section{Main Results: Performance Guarantees} \label{sec:perf}

In this section we study the performance of our algorithm under the semi-random \gpp and provide theoretical guarantees. We give a unified theorem, and then discuss its consequences for various special cases, and compare with literature. We also discuss how to estimate the parameter $ t $ in the special case of the standard stochastic blockmodel. We shall first consider the case with $ p>q $. The $ p<q $ case is similar and is discussed in Section~\ref{sec:heterophily}. All proofs are postponed to Sections~\ref{sec:proof_mono} to~\ref{sec:proof_est_pq}.

\subsection{A Monotone Lemma}
Our optimization-based algorithm has a nice monotone property: adding/removing edges ``aligned with" the optimal~$\widehat{Y}$ (as is done by the adversary under the semi-random setting) cannot result in a different optimal solution. This is summarized in the following lemma.
\begin{lem}\label{lem:monotone}
Suppose $ p>q $ and $\widehat{Y}$ is the unique optimal solution of~(\ref{eq:cvx_prog})--(\ref{eq:inequalities}) for a given $ A $ and $ t $. If now we arbitrarily change some edges of  $A$ to obtain $\widetilde{A}$, by (a) choosing some edges such that $\widehat{y}_{ij} = 1$ but $a_{ij} = 0$, and making $\widetilde{a}_{ij} = 1$, and (b) choosing some edges where $\widehat{y}_{ij} = 0$ but $a_{ij} = 1$, and making $\widetilde{a}_{ij} = 0.$ Then, $\widehat{Y}$ is also the unique optimal solution of~(\ref{eq:cvx_prog})--(\ref{eq:inequalities}) with~$\widetilde{A}$ as the input and the same $ t $.
\end{lem}
The lemma shows that our algorithm is inherently robust under the semi-random model. In particular, the algorithm succeeds in recovering the true clusters on the semi-random \gpp as long as it succeeds on the \gpp with the same parameters. In the sequel, we therefore focus solely on the \gpp, with the understanding that any guarantee for it immediately implies a guarantee for the semi-random variant.

\subsection{Main Theorem}
Let $Y^*$ be the matrix corresponding to the true clusters in the \gpp, i.e., $y_{ij}^* = 1$ if and only if $ i,j \in V_1 $ and they are in the same true cluster, and 0 otherwise. The theorem below establishes conditions under which our algorithm, specifically the convex program (\ref{eq:cvx_prog})--(\ref{eq:inequalities}), yields this $Y^*$ as the unique optimal solution with high probability (without any further need for rounding etc.).

\begin{thm}
\label{thm:main} Suppose the graph $ A $ is generated according to the \gpp with $ p>q $. If $ t $ in~\eqref{eq:weights} is chosen to satisfy 
\begin{equation}
\frac{1}{4}p +\frac{3}{4}q \le t\le \frac{3}{4}p + \frac{1}{4}q,\label{eq:t_cond}
\end{equation}  
then $Y^{*}$ is the unique optimal solution to the convex program~\eqref{eq:cvx_prog}--\eqref{eq:inequalities} with probability at least $1-4 n^{-8}$ provided 
\begin{align}
\frac{p-q}{\sqrt{p(1-q)}}  \ge  c_{1} \max\left\{ \frac{\sqrt{n}}{K}, \frac{\log^2 n}{\sqrt{K}} \right\},\label{eq:main_cond}
\end{align}
where $c_{1}$ is an absolute constant independent of $ p,q,K,r $ and~$ n $.
\end{thm}

Our theorem quantifies the tradeoff between the four parameters governing the hardness of \gpp---the minimum in-cluster edge density $ p $, the maximum across-cluster density $ q $, the minimum cluster size $K$  and the number of unaffiliated nodes $ n_2=n-n_1 $---required for our algorithm to succeed, i.e., to recover the underlying true clustering {without any error}. Note that we can handle any values of $ p,q,n_2 $ and $ K $ as long as they satisfy the condition in the theorem; in particular, they are allowed to scale with $n$. Interestingly, the theorem does not have an explicit dependence on the number of clusters $ r $ (except for the requirement $ rK\le n $). We note that by using a slightly stronger version of the spectral bound in Lemma~\ref{lem:random_matrix} in the appendix (see e.g., \cite{feige2005spectral}), it is possible to improve the $ \log^2 n $ factor in~\eqref{eq:main_cond} to $ \sqrt{\log n} $. We omit such logarithmic improvement for reasons of space.

We now discuss the \emph{tightness} of Theorem~\ref{thm:main} in terms of these model parameters. When the  minimum cluster size $ K=\Theta(n) $, we have a near-matching converse result. 
\begin{thm}
\label{thm:converse}
Suppose  all clusters have equal size $ K $, and the in-cluster (cross-cluster, resp.) edge probabilities are uniformly $ p $ ($ q $, resp.), with $ K=\Theta(n) $ and $ n_2 = \Theta(n_1) $. Under \gpp with $ p>q $ and $ n $ sufficiently large, for any algorithm to correctly recover the clusters with probability at least $ \frac{3}{4} $, we must have 
\begin{align*}
  \frac{p-q}{\sqrt{p(1-q)}} \ge c_2 \frac{1}{\sqrt{n}},
\end{align*} 
where $ c_2 $ is an absolute constant.
\end{thm}
This theorem gives a necessary condition for \emph{any} algorithm to succeed regardless of its computational complexity. It shows that Theorem~\ref{thm:main} is optimal up to logarithmic factors for all values of $ p $ and $ q $ when $ K=\Theta(n)$.

For smaller values of the minimum cluster size $ K $,  Theorem~\ref{thm:main} requires $K$ to be $\Omega(\sqrt{n})$ since the left hand side of (\ref{eq:main_cond}) is less than $ 1 $. This lower-bound is achieved when $ p $ and $ q $ are both constants independent of $ n $ and $ K $. There are reasons to believe that this requirement is unlikely to be improvable using polynomial-time algorithms. Indeed, the special case with $ p=1 $ and $q=\frac{1}{2}$ corresponds to the classical planted clique problem~\cite{alon1998hiddenClique}; finding a clique of size $ K=o(\sqrt{n}) $ is widely believed to be computationally hard even on average and has been used as a hard problem for cryptographic applications~\citep{Hazan2011Nash,Juel00cliqueCrypto}.

For other values of $ p $ and $ q $, no general and rigorous converse result exists. However, there is evidence suggesting that no other polynomial-time algorithm is likely to have better guarantees than our result in~\eqref{eq:main_cond}. The authors of~\cite{decelle2011asymptotic} show, using non-rigorous but deep arguments from statistical physics, that recovering the clustering is impossible in polynomial time if $\frac{p-q}{\sqrt{p}}=o\left( \frac{\sqrt{n}}{K}\right) $. Moreover, the work in~\cite{nadakuditi2012spectra} shows that a large class of spectral algorithms fail under similar conditions. In view of these results, it is possible that our algorithm is order-wise optimal with respect to all polynomial-time algorithms for all values of $ p $, $ q $ and~$ K $.

We give several further remarks regarding Theorem~\ref{thm:main}. 
\begin{itemize}
\item A nice feature of our result is that we only need ${p}-{q}$ to be large {\em as compared to} $\sqrt{{p}}$; several other existing results (see Table~\ref{tab:comparison}) require a lower bound (as a function of  $n$ and $ K $) on ${p}-{q}$ itself.
When $K$ is $\Theta(n) $, we allow ${p} $ and $ p-q $ to be as small as $\Theta\left(\log^4 (n)/n\right)$.

\item The number of clusters $ r $ is allowed to grow rapidly with~$ n $---sometimes called the high-dimensional setting~\cite{rohe2011stochasticblock}. In particular, our algorithm can recover up to $ r ={\Theta}(\sqrt{n}) $ equal-sized clusters when $ p-q=\Theta(1) $. Any algorithm with a better scaling would recover cliques of size $ o(\sqrt{n}) $, an unlikely task in polynomial time in light of the hardness of the planted clique problem discussed above.

\item  The number of unaffiliated nodes can be large, as many as $ n_2 = {\Theta}(n) = {\Theta}(n_1^2) $, which is attained when $ p-q, r $ are $ \Theta(1) $ and the clusters have equal size. In other words, almost all the nodes can be unaffiliated, and this is true even when there are multiple clusters that are not cliques~(i.e., $p<1  $).

\item Not all existing methods can handle non-uniform edge probabilities and node degrees, which often require special treatment (see e.g.,~\cite{chaudhuri2012extended}). This issue is addressed seamlessly by our method by definition of \gpp.
\end{itemize}

\subsection{Consequences and Comparison with Literature}
In this subsection we discuss the consequences of Theorem~\ref{thm:main} for specific planted problems and compare with existing work. Our results match the best existing results in all cases (up to logarithmic factors), and in many important settings lead to order-wise stronger guarantees.

\subsubsection{Standard Stochastic Blockmodel (a.k.a. Planted Partition Model)}
This model assumes that all clusters have the same size $ K $ with no unaffiliated nodes ($ n_2 = 0 $) and $ p>q $. We compared our result to past approaches and theoretical results in Table~\ref{tab:comparison}: For general values of $ p,q $ and $ K $, our result has the scaling $ p-q={\Omega}\left(\frac{\sqrt{pn}}{K}\right) $ and $ p=\Omega\left(\frac{n}{K^2}\right) $, which improves on all existing results by polynomial factors. This means that we can handle much noisier and sparser graphs, especially when the number of clusters $ r=n/K $ is growing.

\subsubsection{Planted $ r $-Disjoint-Cliques Problem}\label{sec:clique}
Here the task is to find a set of $ r $ disjoint cliques, each of size at least $ K $, that have been planted in an Erdos-Renyi random graphs $ G(n,q) $. Setting $ p=1 $ in Theorem~\ref{thm:main}, we obtain the following guarantee for this problem.
\begin{cor}\label{cor:clique}
For the planted $r $-disjoint-cliques problem, the formulation~\eqref{eq:cvx_prog}-\eqref{eq:weights} with $ t $ chosen according to Theorem~\ref{thm:main} finds the hidden cliques with probability at least $1-4 n^{-8}$ provided
\[
1-q\ge c_3 \max\left\{\frac{n}{K^{2}}, \frac{\log^{4}n}{K}\right\}, 
\]
where $ c_3 $ is an absolute constant.
\end{cor}
In the regime where $ r $ is allowed to scale with $ n $ and $ q $ is bounded away from zero, the best previous results are given in~\cite{mcsherry2001spectralpartitioning} with $ 1-q = {\Omega}(\frac{rn}{K^2}) $ and in~\cite{ames2012clustering} with $ 1-q={\Omega}(\frac{\sqrt{n}}{K}) $. Corollary~\ref{cor:clique} is stronger than both of them for large $ r $. In the special case with $ r=1 $ and $ q=1/2 $, which is the standard planted clique problem, the corollary guarantees recovery for the clique size $ K=\Omega(\sqrt{n}) $, matching the best known bound~\cite{alon1998hiddenClique}.

\subsubsection{The $ p<q $ Case}\label{sec:heterophily}
Given a graph $ A $ generated from the semi-random \gpp with in/cross-cluster densities $ p <q $, we can run our algorithm on the graph $ A' = \mathbf{1}\mathbf{1}^\top - A $, where $ \mathbf{1}\mathbf{1}^\top $ is the all-one matrix. Note that $ A' $ can be considered as generated from \gpp with in/cross-cluster densities $ p'=1-p$ and $q'=1-q $, where  $ p'>q' $. With this reduction, Theorem~\ref{thm:main} immediately yields the following guarantee.
\begin{cor}\label{cor:disa}
Under the semi-random \gpp with $ p<q $, the formulation~\eqref{eq:cvx_prog}-\eqref{eq:weights} applied to $  \mathbf{1}\mathbf{1}^\top - A $ with $ t $ satisfying
\[
\frac{3}{4}p+\frac{1}{4}q \le 1-t \le \frac{1}{4}p+\frac{3}{4}q
\]
finds the true clustering with probability  $1-4 n^{-8}$ provided
\[
q-p\ge c_3 \sqrt{(1-p)q}\max\left\{\frac{\sqrt{n}}{K}, \frac{\log^2 n}{\sqrt{K}}\right\}, 
\]
where $ c_3 $ is an absolute constant.
\end{cor}
This corollary implies guarantees for the planted coloring problem~\cite{alon1997coloring3} and the planted $ r $-cut~\cite{bollobas2004maxcut} (a.k.a. planted noisy coloring~\cite{decelle2011asymptotic}) problem. We are not aware of any exiting work that explicitly considers the \gpp with $ p<q $ in its general form (i.e., $ n_2>0 $, $ 1 > q > p > 0 $, and $ K =o(n)$ with potential non-random edges). However, since any guarantee for \gpp with $ p>q $ implies a guarantee for \gpp with $p<q $, Table~\ref{tab:comparison} provides a comparison with existing work when $ n_2 = 0 $ and the edge probabilities and cluster sizes are uniform. Again our guarantee outperforms all existing ones.

\subsubsection{Planted Coloring Problems}
This is a special case of the above setting, where $ p=0$, $ n_2=0 $ and the goal is to find the $ r $ planted groups of colored nodes with no edge between nodes with the same color. The best existing result $ q=\Omega\left(\frac{n}{K^2}+\frac{\log n}{K}\right)$ is achieved by various algorithms; see e.g.,~\cite{alon1997coloring3,coja2004coloringSemirandom}. By Corollary~\ref{cor:disa}, our algorithm succeeds when $ q = \Omega\left( \frac{n}{K^2} + \frac{\log^4n}{K} \right) $. We match the best existing algorithms for $ K=O(n/\log^4(n)) $, and are off by a few log factors for larger $ K $. 

\subsubsection{Clustering Partially Observed Graphs}

In many applications the pairwise relations in the graph are  partially observed, meaning that the values of $ A_{ij} $ are known only for a subset of the pairs $ (i,j) $, and information of other pairs is impossible or too expensive to obtain~\cite{Jalali2011clustering,shamir2011budget}. A standard and natural model for this setting is as follows: after the graph $ A $ is generated according to the \gpp with edge densities $p $ and~$ q $, each entry of $ A $ is erased (i.e., unobserved) independently with probability $ 1-s $, so $ s\in[0,1] $ is the observation probability. One possible approach is to set to zero all the entries of $ A $ that are unobserved, and apply our algorithm to the zero-imputed graph $ A'' $. Note that $ A'' $ can be considered as generated from the \gpp with in/cross-cluster densities equal to $ ps $ and~$ qs $, respectively. Theorem~\ref{thm:main} is powerful enough to imply the following strong guarantee for this simple approach.

\begin{cor}\label{cor:partial}
Under the above setting with $ p>q $, the formulation~\eqref{eq:cvx_prog}-\eqref{eq:weights} applied to $ A'' $ with $ t $ satisfying
\[
\frac{1}{4}ps+\frac{3}{4}qs \le t \le \frac{3}{4}ps+\frac{1}{4}qs
\]
finds the true clustering with probability $1-4 n^{-8}$ provided
\[
(p-q)\sqrt{\frac{s}{p}}\ge c_4 \max\left\{\frac{\sqrt{n}}{K}, \frac{\log^2 n}{\sqrt{K}}\right\}, 
\]
where $ c_4 $ is an absolute constant.
\end{cor}

The work in~\cite{Jalali2011clustering} considers the special case with $ p=1-q>1/2 $. Their algorithm explicitly handles unobserved pairs and requires the condition $ (2p-1)\sqrt{s} \gtrsim \frac{\sqrt{n}\log n}{K},$ which is the best known result in this setting. Corollary~\ref{cor:partial} matches this result up to at most a logarithmic factor, and in addition applies to settings with more general values of $ p $ and $ q $. The algorithm proposed in~\cite{oymak2011finding} also imputes unobserved pairs with zeros and requires $ (p-q)s\gtrsim \max\{\frac{\sqrt{n}}{K}, \sqrt{\frac{\log n}{K}}\}  $. Corollary~\ref{cor:partial} is order-wise better whenever~$ K\lesssim {n}/{\log^4n} $.

\subsection{Estimating $ t $ in Special Cases}\label{sec:est_p}

We argued in Section~\ref{sec:remark_alg} that specifying $ t $  in a completely data-driven way is ill-posed  for the general \gpp, e.g., when the clusters have a hierarchical structure. However, for some cases this can be done reliably with strong guarantees. Consider the standard stochastic blockmodel, where all the $ r=n/K $ clusters have the same size $ K $, the edge probabilities are uniform (i.e., equal to $ p $ within clusters and $ q $ across clusters, with $ p>q $), and there are no unaffiliated nodes ($ n_2 = 0 $) or non-random edges. Without lost of generality, we may re-label the nodes such that the $ l $-th cluster has nodes $ (l-1)K+1, (l-1)K+2, \ldots, lK $. Observe that the matrix $\bar{A} := \mathbb{E}\left[A\right]- (1-p)I$ is a matrix with blocks of $ p $ and $ q $'s,\footnote{Recall that we use the convention $ a_{ii}=1 $.} and therefore can be written as
$$
\bar{A} = \mathbf{1} \mathbf{1}^\top \otimes B,
$$
where $ \mathbf{1} $ is the all one vector in $ \mathbb{R}^K $, and $ B\in\mathbb{R}^{r\times r} $ equals $ p $ on the diagonal and $ q $ elsewhere. In words, $ \bar{A} $ is the Kronecker product of a $ K\times K $ all-one matrix $ \mathbf{1} \mathbf{1}^\top  $ and an $ r\times r $ circulant matrix $ B $.
The matrix $ \mathbf{1} \mathbf{1}^\top  $ has only one non-zero eigenvalue $ K $, and the matrix $ B $ has eigenvalues $ (p-q) + rq $ and $ p-q $ with multiplicities $ 1 $ and $ r-1 $, respectively. The eigenvalues of $\bar{A}$ are the products of the eigenvalues of $ \mathbf{1} \mathbf{1}^\top  $ and $ B $. Since $ n=Kr $, it follows that the eigenvalues of $\mathbb{E}\left[A\right] = \bar{A} + (1-p)I$ are:
$$ 
\begin{cases}
K(p-q)+nq+(1-p)&\text{with multiplicity } 1 ,\\ 
K(p-q)+(1-p)& \text{with multiplicity } r-1, \\
1-p& \text{with multiplicity } n-r;
\end{cases}
$$ 
see~\cite{giesen2005manypartition} for a similar derivation. Given these eigenvalues of $ \mathbb
{E}[A] $, we can compute the values of $ r $ and $ K  $ as there is a gap between the $ r $-th and $ (r+1)$-th eigenvalues, and then solve for $ p $, $ q $  (and therefore $ t $) using the first two eigenvalues. In practice, we use the observed matrix $ A$ instead of $ \mathbb{E}[A] $; see Algorithm~\ref{alg:est_pq}.

\begin{algorithm}[h]
\caption{Estimation of  $ t $ from Data}
\label{alg:est_pq}
	\begin{enumerate}
	\item Compute and sort the eigenvalues of $A$, denoted as $\hat{\lambda}_{1}\ge\hat{\lambda}_2\ge\ldots\ge\hat{\lambda}_{n}$.
	
	\item Let $\hat{r} = \arg \max_{i=2,\ldots,n-1} (\hat{\lambda}_i - \hat{\lambda}_{i+1})$ (break ties arbitrarily).
	Set $\hat{K}=n/\hat{r}$.
	
	\item Set $
	\hat{p}=\frac{\hat{K}\hat{\lambda}_{1}+(n-\hat{K})\hat{\lambda}_{2}-n}{n(\hat{K}-1)}$, $\hat{q} = \frac{\hat{\lambda}_1 - \hat{\lambda}_2}{n}$ and $ t=\frac{\hat{p}+\hat{q}}{2} $.
	\end{enumerate}
\end{algorithm}

The following theorem guarantees that the estimation errors are sufficiently small. 

\begin{thm}
\label{thm:est_pq}
Under the standard stochastic blockmodel and the condition~\eqref{eq:main_cond} in Theorem~\ref{thm:main}, the parameters estimated in Algorithm~\ref{alg:est_pq} satisfy the following with probability at least $1-4 n^{-8}$, where $ c_4 $ is an absolute positive constant:
\begin{align*}
\hat{K} = & K,\\
\max\left\{\left|\hat{p}-p\right|,\left|\hat{q}-q\right|\right\}  \le & c_{4}\frac{\sqrt{p(1-q)n}}{K},\\
t  \in & \left[\frac{1}{4}p +\frac{3}{4}q, \frac{3}{4}p + \frac{1}{4}q\right].
\end{align*}
\end{thm}
In particular, the estimated $ t $ satisfies the condition~\eqref{eq:t_cond} in Theorem~\ref{thm:main}. The above theorem also ensures that Algorithm~\ref{alg:est_pq} is a consistent estimator of the parameters $ p $ and $ q $ when condition~\eqref{eq:main_cond} is satisfied, which may be a result of independent interest.
Combining Theorem~\ref{thm:main} and Theorem~\ref{thm:est_pq}, we obtain a complete algorithm that is guaranteed to find the clusters for the standard stochastic blockmodel under the condition~\eqref{eq:main_cond}, without any knowledge of the parameters of the underlying generative model.

\section{Empirical Results}
\label{sec:expts}
In this section we discuss implementation issues of our algorithm, and provide empirical results on synthetic and real-world datasets.

\subsection{Implementation Issues} \label{sec:implementation}

The convex program~\eqref{eq:cvx_prog}--\eqref{eq:inequalities} can be solved using a general purpose
SDP solver, but this method does not scale well to problems with more
than a few hundred nodes. To facilitate a fast and efficient solution, we propose
to use a family of first-order algorithms called the Augmented Lagrange Multiplier
(ALM) method. Note that the program~\eqref{eq:cvx_prog}--\eqref{eq:inequalities} can be rewritten as
\begin{align}
 \min_{Y,S\in\mathbb{R}^{n\times n}} 
    & \quad  \lambda\Vert C \circ S \Vert_1 + \Vert Y\Vert_*\label{eq:L+S}\\
\mbox{s.t }  &\quad Y+S=A,\nonumber \\
   & \quad 0 \le y_{ij} \le 1, \forall i,j,\nonumber
\end{align}
where $ \lambda  :=\frac{1}{48\sqrt{n}}$, the matrix $ C\in \mathbb{R}^{n \times n} $ satisfies $ c_{ij} = \ca $ if $ a_{ij}=1 $ and $ c_{ij} = \cac $ otherwise, and $ \circ $ denotes the element-wise product. This problem can be recognized as a weighted version of the standard convex formulation of the low-rank and sparse matrix decomposition problem~\cite{candes2009robustPCA,chandrasekaran2011siam}, of which the numerical solution has been well studied. 
We adapt the ALM method in~\cite{Lin2009_ALM} to the above problem as given in Algorithm~\ref{alg:alm}.
\begin{algorithm}
\caption{ALM Method for the Program~\eqref{eq:L+S} of Minimizing Nuclear Norm plus Weighted $ \ell_1 $ Norm \label{alg:alm}}
\begin{algorithmic}
\STATE Input: $A,C\in\mathbb{R}^{n\times n}$. 
\STATE Initialize: $M^{(0)}=0$; $Y^{(0)}=0$; $S^{(0)}=0$; $\mu_{0}>0$;
$\alpha>1$; $k=0$, $ \lambda = \frac{1}{48\sqrt{n}} $.
\WHILE{not converge}
    \STATE $ (U,\Sigma,V)  =  \text{SVD}(A-S^{(k)}+\mu_{k}^{-1}M^{(k)}) $.
    \STATE $ \bar{Y}^{(k+1)}  =  U\mathcal{S}_{\mu_{k}^{-1}}(\Sigma)V $.
    \STATE  For all $ (i,j) $, $y^{(k+1)}_{ij} = \max\left\{ \min \left\{ \bar{Y}^{(k+1)}_{ij}, 1 \right\}  ,0 \right\}$.
    \STATE $ S^{(k+1)}  =  \mathcal{S}_{\mu_{k}^{-1}\lambda C}(A-Y^{(k+1)}+\mu_{k}^{-1}M^{(k)}) $.
    \STATE $ M^{(k+1)}  =  M^{(k)}+\mu_{k}(A-Y^{(k+1)}-S^{(k+1)}) $.
    \STATE $ \mu_{k+1} = \alpha \mu_k $, $ k  =  k+1 $.
\ENDWHILE 
\STATE Return $Y^{(k+1)},S^{(k+1)}$.
\end{algorithmic}
\end{algorithm}
Here $\mathcal{S}_{X}(\cdot):\mathbb{R}^{n\times n}\mapsto\mathbb{R}^{n\times n}$
is the element-wise weighted soft-thresholding operator, defined as
\begin{align*}
\left(\mathcal{S}_{X}(M)\right)_{ij} & =  \begin{cases}
m_{ij}-x_{ij}, & \quad \textrm{if } m_{ij}>x_{ij}\\
m_{ij}+x_{ij}, & \quad \textrm{if } m_{ij}<-x_{ij}\\
0, & \quad\textrm{otherwise,}
\end{cases}
\end{align*}
for any matrices $ M,X\in\mathbb{R}^{n \times n} $.
In other words, it shrinks each entry of $M$ towards zero by $ x_{ij}$. The unweighted version
$\mathcal{S}_{\epsilon}(\cdot) := \mathcal{S}_{\epsilon I}(\cdot)$ is also used. The parameters of the algorithm are set as $ \mu_0 = 1.25/\Vert A \Vert$ and $\alpha = 1.5 $ to suggested by~\cite{Lin2009_ALM}. Following~\cite{Lin2009_ALM}, it can be shown that the ALM method is guaranteed to converge to a global optimal solution.

While~\cite{Lin2009_ALM} does not prove a convergence rate for the ALM method, it is observed there that it converges Q-linearly. We observe a similar behavior, as shown in Figure~\ref{fig:0}. In the subsequent simulations, we use  $ \Vert A-Y^{(k)} - S^{(k)} \Vert_F /\Vert A \Vert_F \le 10^{-2}$ as the stopping criterion, so the number of iteration needed is usually small. The main bottleneck of the algorithm is computing the SVD in each iteration. Therefore, the time complexity of the algorithm is roughly the time for one SVD multiplied by the number of iterations. This can be compared with spectral clustering, which requires one SVD. The memory requirement of the ALM algorithm is $ \Theta(n^2) $, i.e., the same order as the space needed to store the graph.  It is possible to improve the space and time complexity by various approaches, such as only storing sparse and low-rank matrices and computing the first few singular values/vectors instead of a full SVD; see~\cite{Lin2009_ALM} for more discussion on implementation details.

\begin{figure}
\centering
\includegraphics[width= 0.8 \columnwidth]{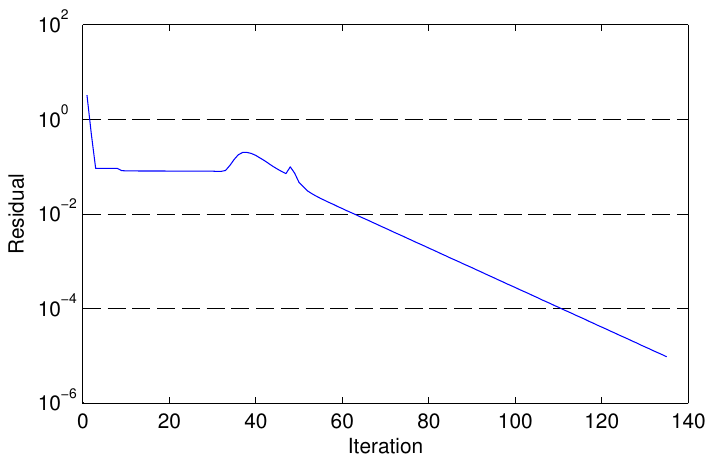}
\caption{\label{fig:0}Convergence of the ALM method. The figure shows the residual $ \Vert A-Y^{(k)} - S^{(k)} \Vert_F /\Vert A \Vert_F$ at each iteration. The plot is generated under the setting with $ n=1000 $ nodes, $ r=5 $ clusters with equal size $ K = 200 $, and $ p=0.35 $, $ q=0.1 $.}
\end{figure}

\subsection{Simulations}
We perform experiments on synthetic data, and compare with other methods. We generate a graph using the stochastic blockmodel with $ n=1000 $ nodes, $r=5$ clusters with equal size $ K=200 $, and $p,q\in [0,1]$. We apply our method to the graph, where we pick $ t $ using Algorithm~\ref{alg:est_pq}
and solve the optimization problem using Algorithm~\ref{alg:alm}.
Due to numerical accuracy, the output $ \hat{Y} $ of our algorithm may not be strictly integer, so we do the following simple rounding: compute the mean $ \bar{y} $ of the entries of $ \hat{Y} $, and round each entry of $ \hat{Y} $ to $ 1 $ if it is greater than $ \bar{y} $, and $ 0 $ otherwise. We measure the error by $ \Vert Y^* - \textrm{round}(\hat{Y}) \Vert_1 $, which equals the number of misclassified pairs. We say our method succeeds if it misclassifies less than $0.1\%$ of the pairs.

For comparison, we consider three alternative methods: (1) Single-Linkage clustering (SLINK)~\cite{sibson1973slink}, which is a hierarchical clustering method that merges the most similar clusters in each iteration. We use the difference of neighbors, namely  $ \Vert A_{i\cdot} - A_{j\cdot} \Vert_1 $, as the distance measure of nodes $ i $ and $ j $, and terminate when SLINK finds a clustering with $ r=5 $ clusters. (2) A spectral clustering method~\cite{von2007spectral}, where we run SLINK on the top $ r=5 $ singular vectors of $ A $. (3) The low-rank-plus-sparse approach~\cite{Jalali2011clustering, oymak2011finding}, followed by the  rounding scheme described in the last paragraph. Note the first two methods assume knowledge of the number of clusters $ r $, which is not available to our method. 

For each value of $ q $, we find the smallest $ p $ for which a method succeeds, and average over $20$ trials. The results are shown in Figure~\ref{fig:1}(a), where the area above each curve corresponds to the range of feasible $ (p,q) $ for each method. It can been seen that our method outperforms all others, in that we succeed for a strictly larger range of $ (p,q) $. Figure~\ref{fig:1}(b) shows more detailed results for sparse graphs ($ p\le 0.3, q\le 0.1 $), for which SLINK and the low-rank-plus-sparse approach  completely fail, while our method significantly outperforms the spectral method, the only alternative method that works in this regime. The running time of each method is shown in Figure~\ref{fig:1} (c). Our approach and the low-rank-plus-sparse approach (both based on convex optimization) require more computational time than the simpler spectral method and SLINK. This suggests a tradeoff between the statistical and computational performance of clustering algorithms.

\begin{figure*}
\centering
\begin{tabular}{ccc}
\includegraphics[width=0.63\columnwidth]{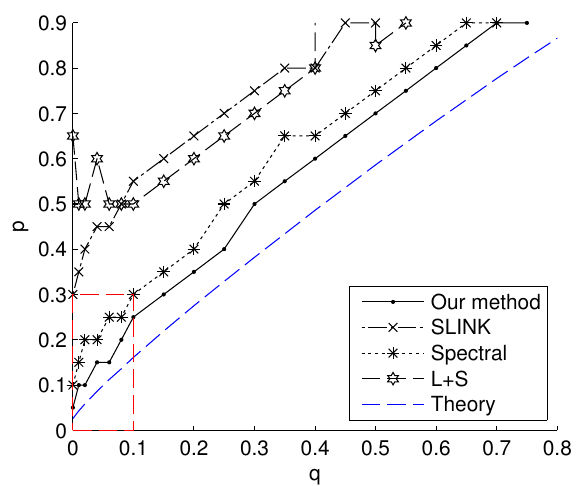} 
&\includegraphics[width=0.63\columnwidth]{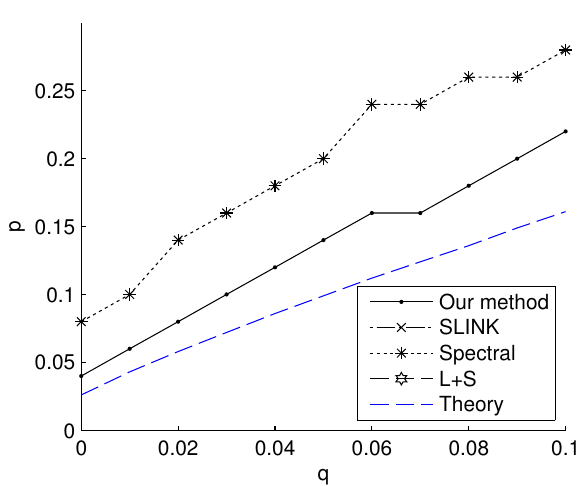} 
&\includegraphics[width=0.73\columnwidth]{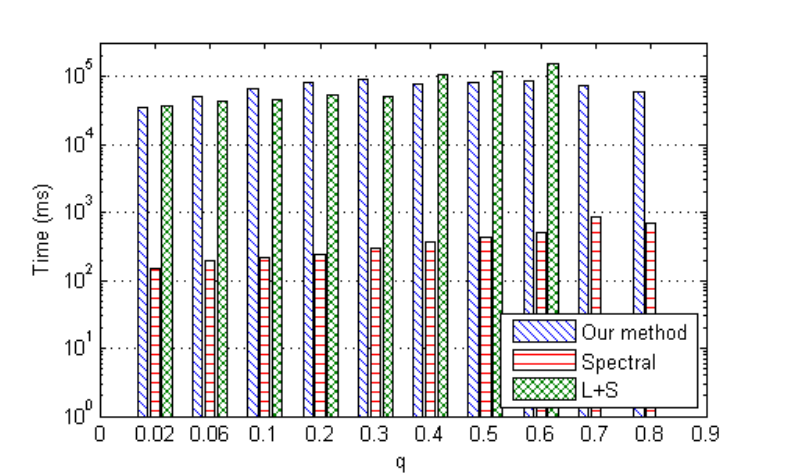}\\
(a)&(b)&(c)
\end{tabular}

\caption{\label{fig:1}Comparison of our method with Single-Linkage clustering (SLINK), spectral clustering, and low-rank-plus-sparse (L+S) approach. (a) The area above each curve is the values of $ (p,q) $ for which a method successfully recovers the underlying true clusters. The dash line corresponds to the bound $ p-q \ge {\sqrt{p(1-q)n}}/{K}$ predicted by our theoretical result in Theorem~\ref{thm:main}. (b) More detailed results for the area in the box in (a), corresponding to sparse graphs. (c) Running times (in milliseconds) for each methods running on different values of $ q $ and the smallest $ p $ for which the method succeeds. A missing bar means a method fails for any $ p $. The running time of the SLINK method is negligible compared to the other methods and is thus displayed in the plot. The experiments are conducted on synthetic data with $ n=1000 $ nodes and $ r=5 $ clusters with equal size $ K = 200 $, using a computer with a Pentium Dual-Core 3.2GHz CPU and 4.00 GB memory.}
\end{figure*}

\subsection{Real-world Collaboration Graph}

We evaluate our method on the NIPS Conference Papers Vol. 0-12 Dataset.\footnote{Available at \url{http://www.cs.nyu.edu/~roweis/data.html}} It contains the authorship relation of $ 2037 $ authors and $ 1740 $ papers. We use this dataset to generate a $ 2037 \times 2037$ graph of the authors by connecting co-authors; that is, we place an edge between a pair of authors if they have written at least one NIPS paper together. This is a sparse graph with an overall edge density of $ 0.002 $.

We apply the four methods to this graph and compare their performance. For fairness, we force all methods to partition the authors into $ r=8 $ clusters as follows: the SLINK and spectral algorithms are the same as in the previous sub-section;  for our method and the low-rank-plus-sparse approach, we apply SLINK to their output $ \hat{Y} $ with $ \Vert \hat{Y}_{i\cdot}  - \hat{Y}_{j\cdot}\Vert_1 $ as the distance measure to obtain $ 8 $ clusters; the parameter $ t $ for our method is estimated using Algorithm~\ref{alg:est_pq} with $ \hat{r} $ fixed to $ 8 $. We measure the quality of the solutions by computing the in-cluster and cross-cluster edge densities, which
are shown in Table~\ref{table:nips}. The clustering produced by
our method has higher in-cluster density and lower cross-cluster density.

\begin{table}
\caption{\label{table:nips}Clustering quality on the NIPS datasets}
\centering

\begin{tabular}{rcc}
\hline & In-Cluster edge density & Cross-cluster edge density\\
\hline Our method & 109$ \times 10^{-4} $& 1.83$ \times 10^{-4} $\\
       SLINK & 26$ \times 10^{-4} $ & 3.42$ \times 10^{-4} $\\
       Spectral & 64$ \times 10^{-4} $ & 1.88$ \times 10^{-4} $\\
       L+S & 86$ \times 10^{-4} $ & 6.07$ \times 10^{-4} $\\
\hline
\end{tabular} 

\end{table}

\section{Proof of Lemma~\ref{lem:monotone}}
\label{sec:proof_mono}

In this section we establish the monotone lemma. Set $\lambda=\frac{1}{48\sqrt{n}}$. Define $\Omega_{+}=\{ (i,j):a_{ij}=0,\tilde{a}_{ij}=1\} $
and $\Omega_{-}=\{ (i,j):a_{ij}=1,\tilde{a}_{ij}=0\} $.
Let $Y\neq \widehat{Y}$ be an arbitrary alternative feasible solution obeying $ 0\le y_{ij} \le 1, \forall i,j $. By optimality of $\widehat{Y}$ to the original program,
we have
\begin{align}\label{eq:monotone1}
 &\left(\ca\!\sum_{a_{ij}=1}\!\hat{y}_{ij}\!-\!\cac\!\sum_{a_{ij}=0}\!\hat{y}_{ij}\right)-\frac{1}{\lambda}\left\Vert \widehat{Y}\right\Vert_{*}\nonumber\\
>&\left(\ca\!\sum_{a_{ij}=1}\!y_{ij}\!-\!\cac\!\sum_{a_{ij}=0}\!y_{ij}\right)-\frac{1}{\lambda}\left\Vert Y\right\Vert _{*}.
\end{align}
Next, by definition of $\tilde{A}$, $ \Omega_+ $ and $ \Omega_- $, we have
\begin{align}\label{eq:monotone2}
&\left(\ca\!\sum_{\tilde{a}_{ij}=1}\!\hat{y}_{ij}\!-\!\cac\!\sum_{\tilde{a}_{ij}=0}\!\hat{y}_{ij}\right)\!-\!\left(\ca\!\sum_{a_{ij}=1}\!\hat{y}_{ij}\!-\!\cac\sum_{a_{ij}=0}\!\hat{y}_{ij}\right)\nonumber\\
=&\!\sum_{(i,j)\in\Omega_{+}}(\ca+\cac);
\end{align}
and
\begin{align}\label{eq:monotone3}
   & \left(\ca\!\sum_{a_{ij}=1}\!y_{ij}\!-\!\cac\!\sum_{a_{ij}=0}\!y_{ij}\right)\!-\!\left(\ca\!\sum_{\tilde{a}_{ij}=1}\!y_{ij}\!-\!\cac\!\sum_{\tilde{a}_{ij}=0}\!y_{ij}\right)\nonumber\\
  = & (\ca+\cac)\!\sum_{(i,j)\in\Omega_{-}}\!y_{ij}-(\ca+\cac)\!\sum_{(i,j)\in\Omega_{+}}\!y_{ij}\nonumber\\
  \ge & -\!\sum_{(i,j)\in\Omega_{+}}\!(\ca+\cac),
\end{align}
where we use $0\le y_{ij}\le1$ for all $(i,j)$ in the last inequality.
Summing the L.H.S. and R.H.S. of~\eqref{eq:monotone1}--\eqref{eq:monotone3} establishes that
\begin{align*}
&\left(\ca\!\sum_{\tilde{a}_{ij}=1}\!\hat{y}_{ij}\!-\!\cac\!\sum_{\tilde{a}_{ij}=0}\!\hat{y}_{ij}\right)-\frac{1}{\lambda}\left\Vert \widehat{Y}\right\Vert_{*}\\
>&\left(\ca\!\sum_{\tilde{a}_{ij}=1}\!y_{ij}\!-\!\cac\!\sum_{\tilde{a}_{ij}=0}\!y_{ij}\right)-\frac{1}{\lambda}\left\Vert Y\right\Vert_{*}.
\end{align*}
Since $Y$ is arbitrary, we conclude that $ \widehat{Y} $ is the unique optimal solution to the modified program.

\section{Proof of Theorem~\ref{thm:main}}
\label{sec:proof_main}

We prove our main theorem in this section. In the remainder of the paper, \textit{with high probability} (w.h.p.) means with probability at least $1-4 n^{-12}$. The proof consists of three main steps, which we elaborate below.

\subsection{Step 1: Reduction to Homogeneous Edge Probabilities}\label{sec:homogeneous}

We show that it suffices to assume that the in-cluster edge probability is uniformly $p$, and the across-cluster edge probability is uniformly $q$. In the heterogeneous model, suppose an edge is placed between nodes $i$ and $j$
with probability $p_{ij}$ if they are in the same cluster,
where $p_{ij}\ge p$. This is equivalent to the following
two-step model: first flip a coin with head probability $p$,
and add the edge if it is head; if it is tail, then flip another coin
and add the edge with probability $\frac{p_{ij}-p}{1-p}$. By the
monotone property in Lemma~\ref{lem:monotone}, we know that if our convex program succeeds on the graph generated in the first step, then it also succeeds for the second step, because more in-cluster edges are added. Similarly, an across-cluster edge with probability $ q_{ij}\le q $ can be generated equivalently as follows: (1) add an edge with probability~$ q $; (2) if an edge is added in the first step, remove it with probability $ \frac{q-q_{ij}}{q} $. Monotonicity can then be applied.  Therefore, heterogeneous edge probabilities only make the probability of success higher, and thus we only need to prove the homogeneous case.

\subsection{Step 2: Optimality Condition}\label{sec:opt_cond}

We need some additional notation. Both $ m_{ij} $ and $ (M)_{ij} $ denote the $ (i,j) $-th entry of the matrix $ M $, and $ \langle M,N\rangle := \text{trace}(M^\top N) $ is the inner product between two matrices $ M $ and $ N $ with the same size. Four matrix norms are used: the spectral norm $ \Vert M\Vert $ (the largest singular value of $ M $), the nuclear norm $ \Vert M \Vert_* $ (the sum of the singular values of $ M $), the matrix $ \ell_\infty $ norm $ \Vert M \Vert_{\infty} := \max_{i,j} \vert m_{ij}\vert $, and the matrix $ \ell_1 $ norm $  \Vert M \Vert_1 := \sum_{i,j} \vert m_{ij}\vert $. We denote the singular value decomposition of  $Y^{*}$ (notice $Y^*$ is symmetric and positive definite) by $U_{0}\Sigma_{0}U_{0}^{\top}$. For any matrix $ M $, we define $P_{T}(M) := U_0U_0^\top M + MU_0U_0^\top - U_0U_0^\top MU_0U_0^\top $.  For a set $ \Omega $ of matrix indices, let $ P_\Omega(M) $ be the matrix obtained by setting the entries of $ M $ outside $ \Omega $ to zero, and we use $ \sum_{\Omega} $ as a shorthand of $ \sum_{(i,j)\in\Omega}$. Define the sets  $\mca:=\mathrm{support}(A)  $ and $ R  := \mathrm{support}(Y^*) = \mathrm{support}(U_0U_0^\top) $. 

The true cluster matrix $ Y^* $ is an optimal solution to the program~\eqref{eq:cvx_prog}--\eqref{eq:inequalities} if
\begin{equation}\label{eq:basic_opt}
\lambda \ca \sum_{\mca} (y^*_{ij} - y_{ij}) - \lambda \cac\sum_{\mca^c}(y^*_{ij} - y_{ij}) - (\Vert Y^*\Vert_* - \Vert Y\Vert_*) \ge 0
\end{equation}
for all feasible $ Y $ obeying~\eqref{eq:inequalities}. Suppose there is a matrix $ W $ that satisfies
\begin{align}\label{eq:W_subgrad}
\Vert W \Vert \le 1,
P_T(W) = 0.
\end{align}
The matrix $ U_0U_0^\top + W$ is a subgradient of $ f(X) = \Vert X\Vert_* $ at $ X=Y^* $, so $ \Vert Y\Vert_* - \Vert Y^*\Vert_* \ge \langle U_0U_0^\top +W, Y-Y^*\rangle $ for all~$ Y $. Then, we see that~\eqref{eq:basic_opt} is implied by
\begin{align}\label{eq:basic_opt2}
&\lambda \ca \!\sum_{\mca} (y^*_{ij} \!-\! y_{ij}) \!-\! \lambda \cac\!\sum_{\mca^c}(y^*_{ij} \!-\! y_{ij}) \!+\! \langle U_0U_0^\top\!+\!W, Y\!-\!Y^*\rangle \nonumber\\
\ge&  0,\quad\forall Y \in \left\{X: 0\le X_{ij} \le 1, \forall (i,j)\right\}.
\end{align}
The above inequality holds in particular for any feasible $ Y $ of the form  $ Y=Y^* - e_i e_j^\top $ with $ (i,j) \in R$ or $ Y=Y^* + e_i e_j^\top $ with $ (i,j) \in R^c$. This leads to the following element-wise inequalities:
\begin{equation}\label{eq:W_exact}
\begin{aligned}
-\lambda \cac - (U_0U_0^\top + W)_{ij} & \ge 0, \quad\forall (i,j)\in R\cap\mca^c,\\
-\lambda \ca + w_{ij} & \ge 0, \quad\forall (i,j)\in R^c\cap\mca,\\
\lambda \ca - (U_0U_0^\top + W)_{ij} & \ge 0, \quad\forall (i,j)\in R\cap\mca,\\
\lambda \cac + w_{ij} & \ge 0, \quad\forall (i,j)\in R^c\cap\mca^c.
\end{aligned}
\end{equation}
It is easy to see that these inequalities are actually equivalent to~\eqref{eq:basic_opt2}, so together with~\eqref{eq:W_subgrad} they form a sufficient condition for the optimality of $ Y^* $.

Finding a ``dual certificate'' $ W $ obeying the exact conditions~\eqref{eq:W_subgrad} and~\eqref{eq:W_exact} is difficult, and does not guarantee uniqueness of the optimum. Instead, we consider an alternative sufficient condition that only requires a $ W $ that \emph{approximately} satisfies the exact conditions. This is done in Proposition~\ref{lem:opt_condition} below (proved in Section~\ref{sec:proof_opt_cond}), which significantly simplifies the construction of $ W $. Note that condition (b) in the proposition is a relaxation of the equality in~\eqref{eq:W_subgrad}, whereas condition (c) tightens~\eqref{eq:W_exact}. Setting $ \epsilon=0 $ and changing equalities to inequalities in the proposition recover the exact conditions.
\begin{prop}
 \label{lem:opt_condition} 
$Y^{*}$ is the unique optimal solution to the program~\eqref{eq:cvx_prog}--\eqref{eq:inequalities}, if there exists a matrix $W\in\mathbb{R}^{n\times n}$
and a number $0< \epsilon<1$ that satisfy the following conditions: (a) 
$\left\Vert W\right\Vert \le1$,
(b) $\left\Vert P_{T}(W)\right\Vert _{\infty}\le\frac{\epsilon}{2}\lambda\min\left\{ \cac,\ca\right\} $,
and (c)
\begin{equation*}\label{eq:W_approx}
\begin{aligned}
-(1+\epsilon)\lambda \cac - (U_0U_0^\top +W)_{ij} &= 0, \quad\forall (i,j)\in R\cap\mca^c,\\
  -(1+\epsilon)\lambda \ca + w_{ij} &= 0, \quad\forall (i,j)\in R^c\cap\mca,\\
  (1-\epsilon)\lambda \ca - (U_0U_0^\top +W)_{ij} &\ge 0, \quad\forall (i,j)\in R\cap\mca,\\
  (1-\epsilon)\lambda \cac +w_{ij} &\ge 0, \quad\forall (i,j)\in R^c\cap\mca^c.
\end{aligned}
\end{equation*}
\end{prop}

\subsection{Step 3: Constructing $W$}\label{sec:constructW}
 
We build a $W$ that satisfies
the conditions in Proposition~\ref{lem:opt_condition} w.h.p. We use $ \mathbf{1} $ to denote the all-one column vector in $ \mathbb{R}^n $, so $\mathbf{11}^\top $ is the all-one $ n\times n $ matrix.  Let
$\H :=\{(i,i),i=1,\ldots,n\}$ be the set of diagonal entries. For an $ \epsilon$ to be specified later, we define
$W=W_{1}+W_{2}+W_{3}+W_{4}$ with $W_{i}$ given by 
\begin{align*}
W_{1}  = & -P_{R\cap\mca^{c}}(U_{0}U_{0}^{\top})+\frac{1-p}{p}P_{R\cap\mca}(U_{0}U_{0}^{\top}),\\
W_{2}  = &
(1+\epsilon)\lambda \cac\left[-P_{R\cap\mca^{c}}(\mathbf{11^{\top}})+\frac{1-p}{p}P_{R\cap\mca}(\mathbf{11^{\top}})\right],\\
W_{3}  = &
(1+\epsilon)\lambda \ca\left[P_{(R^{c}\cap \H^{c})\cap\mca}(\mathbf{11^{\top}})-\frac{q}{1-q}P_{(R^{c}\cap \H^{c})\cap\mca^{c}}(\mathbf{11^{\top}})\right],\\
W_{4}  = & (1+\epsilon)\lambda \ca P_{R^c}(I),
\end{align*}
where $ I $ is the identity matrix. We briefly explain the ideas behind the construction. Each of the matrices $ W_1 $, $ W_2 $ and $ W_3 $ is the sum of two terms. The first term is derived from the equalities in condition~(c) in Proposition~\ref{lem:opt_condition}. The second term is constructed in such a way that each $ W_i $ is a  zero-mean random matrix (due to the randomness in the set $ \mca $), so it is likely to have small norms and satisfy conditions~(a) and~(b). The matrix $W_{4}$ accounts for the unaffiliated nodes. In particular, it is a diagonal matrix with $(W_{4})_{ii}$ being non-zero if and only if $ i\in V_2 $

The following proposition (proved in Section~\ref{sec:proof_WisGood}) shows that $W$ indeed satisfies all the desired conditions w.h.p., hence establishing Theorem~\ref{thm:main}.
\begin{prop}
\label{prop:WisGood} Under the conditions in Theorem~\ref{thm:main},  $W$ constructed above satisfies the conditions (a)--(c) in Proposition~\ref{lem:opt_condition} w.h.p. with 
$$\epsilon:=\frac{48}{\sqrt{t(1-t)}}\max\left\{ \frac{\sqrt{n}}{K},\sqrt{\frac{\log^{4}n}{K}}\right\}. $$ 
\end{prop}

\subsection{Proof of Proposition~\ref{lem:opt_condition} (Optimality Condition)}
\label{sec:proof_opt_cond}
 
Let $P_{T^\bot} (W) := W-P_{T}(W)$. Consider any feasible solution $Y$ and 
let $ \D:= Y-Y^* $. The difference between the objective values of $Y$ and $Y^{*}$ satisfies
\begin{align}
(*)  := & \ca\sum_{\mca}\d_{ij}-\cac\sum_{\mca^{c}}\d_{ij}-\frac{1}{\lambda}\left\Vert Y^*+\D\right\Vert_{*}+\frac{1}{\lambda}\left\Vert Y^*\right\Vert_{*}\nonumber\\
 \le & \ca\sum_{\mca}\d_{ij}-\cac\sum_{\mca^{c}}\d_{ij}-\frac{1}{\lambda}\left\langle U_{0}U_{0}^{\top}+P_{T^{\bot}}(W),\D\right\rangle \nonumber\\
 = & \ca\sum_{\mca}\d_{ij}-\cac\sum_{\mca^{c}}\d_{ij}-\!\frac{1}{\lambda}\!\left\langle U_{0}U_{0}^{\top}\!+\!W,\!\D\right\rangle \!+\! \frac{1}{\lambda}\!\left\langle P_{T}W,\!\D\right\rangle,\label{eq:obj_diff}
\end{align}
where in the inequality we use the fact that $U_{0}U_{0}^{\top}+P_{T^{\bot}}(W)$
is a subgradient of $\left\Vert Y\right\Vert _{*}$ at $Y{}^{*}$, a consequence of condition (a) in the proposition and $ \Vert P_{T^\bot}(W)\Vert\le \Vert W\Vert $.
We substitute the condition (c) into the third term in~\eqref{eq:obj_diff} to obtain 
\begin{align*}
(*) &\le  
\epsilon\ca\!\sum_{R\cap\mca}\d_{ij} -\epsilon\cac\!\!\sum_{R^c\cap\mca^{c}}\!\!\d_{ij} +\epsilon\cac\!\!\sum_{R\cap\mca^c}\!\!\d_{ij} -\epsilon\ca\!\sum_{R^{c}\cap\mca}\!\d_{ij} \\
&\quad+\frac{1}{\lambda}\left\langle P_{T}W,\D\right\rangle\\
&\le -\epsilon\min\{\ca,\cac\}\Vert \D\Vert_1 +\frac{1}{\lambda}\left\langle P_{T}W,\D\right\rangle,
\end{align*}
where we used the fact that $ \d_{ij}\le 0 $ for $ (i,j)\in R $ and $ \d_{ij}\ge 0 $ for $ (i,j)\in R^c $ since $ Y=Y^*+\D $ satisfies~\eqref{eq:inequalities}. Applying condition (b) yields
\begin{align*}
(*)  
&\le  -\epsilon\min\left\{ \cac,\ca\right\} \left\Vert \D\right\Vert _{1}+\frac{1}{\lambda}\left\Vert P_{T}W\right\Vert _{\infty}\left\Vert \D\right\Vert _{1}\\
&\le  -\frac{\epsilon}{2}\min\left\{ \cac,\ca\right\} \left\Vert \D\right\Vert _{1}.
\end{align*}
The last R.H.S. is strictly negative whenever $ \D \neq 0 $. This proves that
$Y^{*}$ is the unique optimal solution.

\subsection{Proof of Proposition~\ref{prop:WisGood}}
\label{sec:proof_WisGood}

We show that $ W $ constructed in Section~\ref{sec:constructW} satisfies the conditions in Proposition~\ref{lem:opt_condition} w.h.p.
We need two technical lemmas. First notice that the conditions~\eqref{eq:t_cond} and~\eqref{eq:main_cond} in Theorem~\ref{thm:main} imply bounds on various quantities.
\begin{lem}
\label{lem:useful_bounds} Under conditions~(\ref{eq:t_cond}) and~(\ref{eq:main_cond}) in Theorem~\ref{thm:main},
we have   $p(1-q)\ge t(1-t)\ge c\max\left\{ \frac{n}{K^{2}},\frac{\log^{4}n}{K}\right\} $ and  $\epsilon<\frac{1}{2}$.
\end{lem}
\begin{proof} Since $1>t>0$, we have $t(1-t)\ge\frac{1}{2}\min\{t,1-t\}$.
Under condition~(\ref{eq:t_cond}) on $t$, we further have $\min\left\{ t,1-t\right\} \ge\frac{1}{4}(p-q)$
and $p(1-q)\ge t(1-t)$. It then follows from condition~(\ref{eq:main_cond}) that 
\[
t(1-t)\ge\frac{1}{8}(p-q)\ge c'\sqrt{t(1-t)}\max\left\{ \frac{\sqrt{n}}{K},\sqrt{\frac{\log^{4}n}{K}}\right\} ,
\]
which implies the inequalities in part~(1) of the lemma. Part~(2) follows directly
from part~(1) and the definition of $\epsilon$.
\end{proof} 

Due to the randomness of $\mca$, $W_{1}$, $W_{2}$ and $W_{3}$
are symmetric random matrices with independent zero-mean entries.
The support and variance of their entries are bounded in the following
lemma. 
\begin{lem}
\label{lem:somepropertiesofW} The following holds under the \gpp and the conditions~(\ref{eq:t_cond}) and~(\ref{eq:main_cond}) in Theorem~\ref{thm:main}.
\begin{enumerate}
\item For $i=1,2,3$, the absolute values of the entries of $W_{i}$ are bounded
by $B_{i}$ a.s. and their variance is bounded by $\sigma_{i}^{2}$, where
\begin{align*}
B_{1}&:=\frac{1}{pK},   && \sigma_{1}^{2}:=\frac{1}{pK^{2}},\\
B_{2}&:=\frac{2}{p}\lambda\cac,   && \sigma_{2}^{2}:=\frac{4(1-t)}{p}\lambda^2\cac^{2},\\
B_{3}&:=\frac{2}{1-q}\lambda\ca,   && \sigma_{3}^{2}:=\frac{4t}{1-q}\lambda^2\ca^{2}.
\end{align*}

\item We have $\sigma_{i}\ge c\frac{B_{i}\log^{2}n}{\sqrt{K}}$
for $i=1,2,3$.
\end{enumerate}
\end{lem}
\begin{proof}
The first part of the lemma follows from the definitions of the $W_{i}$'s, $ q\le t\le p $ and  
$\epsilon<\frac{1}{2}$ (Lemma~\ref{lem:useful_bounds}). The second part follows from  Lemma~\ref{lem:useful_bounds}. 
\end{proof}

We now proceed with the proof of Proposition~\ref{prop:WisGood}, The proof has three sub-steps, corresponding to checking the three conditions in Proposition~\ref{lem:opt_condition}.

\textbf{(1) Bounding $\left\Vert W\right\Vert $.}

Recall that $W_{1}$ is a random matrix with i.i.d. entries, and their
absolute values and variance are bounded in Lemma~\ref{lem:somepropertiesofW}.
We apply standard bounds on the spectral norm of random matrices
(Lemma~\ref{lem:random_matrix} in the Appendix) to obtain w.h.p. 
\[
\left\Vert W_{1}\right\Vert \le6\frac{1}{K}\sqrt{\frac{1}{p}}\sqrt{n}\le\frac{1}{4},
\]
where the last inequality follows from $p\ge c\frac{n}{K^{2}}$ (cf. Lemma~\ref{lem:useful_bounds}). In a similar manner, we obtain that w.h.p.
\begin{align*}
\left\Vert W_{2}\right\Vert &\le6\cdot2\sqrt{\frac{1-t}{p}}\lambda\cac\cdot\sqrt{n}\\
&=12\sqrt{\frac{(1-t)}{p}}\cdot\frac{1}{48}\sqrt{\frac{t}{(1-t)n}}\cdot\sqrt{n}\le\frac{1}{4},
\end{align*}
where the last inequality follows from $p\ge t$, and w.h.p. 
\begin{align*}
\left\Vert W_{3}\right\Vert &\le6\cdot2\sqrt{\frac{t}{1-q}}\lambda\ca\cdot\sqrt{n}\\
&=12\sqrt{\frac{t}{1-q}}\cdot\frac{1}{48}\sqrt{\frac{1-t}{tn}}\cdot\sqrt{n}\le\frac{1}{4},
\end{align*}
where the last inequality follows from $1-t\le1-q$. Finally, since $W_{4}=(1+\epsilon)\lambda\ca P_{R^c}(I)$
is a diagonal matrix, we have 
\[
\Vert W_{4}\Vert\le(1+\epsilon)\lambda\ca\le2\cdot\frac{1}{48}\sqrt{\frac{1-t}{tn}}\le\frac{1}{4}
\]
since $t\ge c\frac{1}{n}$ (cf. Lemma~\ref{lem:useful_bounds}). We
conclude that $\Vert W\Vert\le\sum_{i=1}^{4}\Vert W_{i}\Vert\le1$.

\textbf{(2) Bounding $\left\Vert P_{T}W\right\Vert _{\infty}$.}

Define the sets $R_{m}:= \{(i,j):i,j \mbox{ in cluster }m\}$, and recall that $ r $ is the number of clusters and $R:= \mathrm{support}(Y^*) = \bigcup_{m=1}^{r}R_{m}$. We have $Y^{*}=\sum_{m=1}^{r}P_{R_{m}}(\mathbf{11}^\top)$, and thus its singular vectors satisfies 
$$U_{0}U_{0}^{\top}=\sum_{m=1}^{r}\frac{1}{k_{m}}P_{R_{m}}(\mathbf{11}^\top).$$Therefore, for $i=1,2,3$,  each entry
of the matrix $U_{0}U_{0}^{\top}W_{i}$
equals $\frac{1}{k_{m}}$ times the sum of $k_{m}$ independent zero-mean
random variables (which are the entries of $ W_i $), whose absolute values and variance are bounded in Lemma~\ref{lem:somepropertiesofW}. Therefore, $\left\Vert U_{0}U_{0}^{\top}W_{i}\right\Vert_{\infty}$ can be bounded by applying the standard Bernstein inequality
(given as Lemma~\ref{lem:subgaussian_concentr} in the Appendix) to each entry of $ U_{0}U_{0}^{\top}W_{i} $ and then using the union bound over all the entries. More specifically, by choosing the constant $ c_0 $ in Lemma~\ref{lem:subgaussian_concentr} sufficiently large such that $ c_1$ in the same lemma is at least $14 $, we have the following:

\begin{itemize}
\item The matrix $W_{1}$ satisfies
\begin{align*}
\left\Vert U_{0}U_{0}^{\top}W_{1}\right\Vert _{\infty}
&\le\frac{1}{K}\cdot c_{0}\sqrt{\frac{1}{pK^{2}}}\sqrt{K\log n}\\
&=c_{0}\frac{1}{K}\sqrt{\frac{\log n}{pK}}\le\frac{\log^{2}n}{24^{2}}\sqrt{\frac{1}{Kn}} \quad\text{w.h.p.,}
\end{align*}
where we use $p\ge c\frac{n}{K^{2}}$ in the last inequality (cf.
Lemma~\ref{lem:useful_bounds}) with $ c $  sufficiently large. 
\item Similarly, the matrix $W_{2}$ satisfies 
\begin{align*}
\left\Vert U_{0}U_{0}^{\top}W_{2}\right\Vert _{\infty}
\le&\frac{1}{K}\cdot c_{0}\sqrt{\frac{1-t}{p}}\lambda\cac\sqrt{K\log n}\\
=& c_{0}\sqrt{\frac{(1-t)\log n}{pK}}\frac{1}{48}\sqrt{\frac{t}{(1-t)n}}\\
\le&\frac{\log^{2}n}{24^{2}}\sqrt{\frac{1}{Kn}} \quad\text{w.h.p.,}
\end{align*}
where we use $p\ge t$ and $\log n$ being sufficiently large in the
last inequality. 
\item  The matrix $W_{3}$ satisfies
\begin{align*}
\left\Vert U_{0}U_{0}^{\top}W_{3}\right\Vert _{\infty}
\le&\frac{1}{K}\cdot c_{0}\sqrt{\frac{t}{1-q}}\lambda\ca\sqrt{K\log n}\\
=& c_{0}\sqrt{\frac{t\log n}{(1-q)K}}\frac{1}{48}\sqrt{\frac{1-t}{tn}}\\
\le&\frac{\log^{2}n}{24^{2}}\sqrt{\frac{1}{Kn}} \quad\text{w.h.p.,}
\end{align*}
where we use $1-q\ge1-t$ and $\log n$ being sufficiently in the
last inequality. 
\item Finally, since $W_{4}$ is a diagonal matrix supported
on $R^{c}$ and $U_{0}U_{0}^{\top}$ is supported on $R$, we have
$U_{0}U_{0}^{\top}W_{4}=0$.
\end{itemize}

On the other hand, we have 
\begin{align*}
\lambda\ca\epsilon
&\ge\frac{1}{48}\sqrt{\frac{1-t}{tn}}\cdot48\sqrt{\frac{\log^{4}n}{Kt(1-t)}}\\
&=\frac{1}{t}\sqrt{\frac{\log^{4}n}{Kn}}
\ge\frac{1}{24}\sqrt{\frac{\log^{4}n}{Kn}}
\end{align*}
and 
\begin{align*}
\lambda\cac\epsilon
&\ge\frac{1}{48}\sqrt{\frac{t}{(1-t)n}}\cdot48\sqrt{\frac{\log^{4}n}{Kt(1-t)}}\\
&=\frac{1}{(1-t)}\sqrt{\frac{\log^{4}n}{Kn}}
\ge\frac{1}{24}\sqrt{\frac{\log^{4}n}{Kn}},
\end{align*}
which implies $$\frac{1}{24}\epsilon\lambda\min\left\{ c_{A},c_{A^{c}}\right\} \ge\frac{\log^{2}n}{24^{2}}\sqrt{\frac{1}{Kn}}.$$
Combining with the previous bounds on $\left\Vert U_{0}U_{0}^{\top}W_{i}\right\Vert _{\infty}$ for $i=1,2,3,4$,
we obtain $\left\Vert (U_{0}U_{0}^{\top}W_{i})\right\Vert _{\infty}\le\frac{1}{24}\epsilon\lambda\min\left\{ c_{A},c_{A^{c}}\right\} .$

Now observe that since $W$ and $U_{0}U_{0}^{\top}$ are both symmetric,
we have $WU_{0}U_{0}^{\top}=\left(U_{0}U_{0}^{\top}W\right)^{\top}$.
Furthermore, we have 
\begin{align*}
\left\Vert U_{0}U_{0}^{\top}WU_{0}U_{0}^{\top}\right\Vert _{\infty}
&\le\left\Vert U_{0}U_{0}^{\top}W\right\Vert _{\infty}\max_j \sum_i \left|\left(U_0U_0^\top\right)_{ij}\right|\\
&\le\left\Vert U_{0}U_{0}^{\top}W\right\Vert _{\infty} .
\end{align*}
It follows that 
\begin{align*}
&\left\Vert P_{T}W\right\Vert _{\infty} \\
  = & \left\Vert U_{0}U_{0}^{\top}W+WU_{0}U_{0}^{\top}-U_{0}U_{0}^{\top}WU_{0}U_{0}^\top\right\Vert _{\infty}\\
  \le & \left\Vert U_{0}U_{0}^{\top}W\right\Vert _{\infty}+\left\Vert WU_{0}U_{0}^{\top}\right\Vert_\infty +\left\Vert U_{0}U_{0}^{\top}WU_{0}U_{0}^\top\right\Vert _{\infty}\\
  \le & 3\left\Vert U_{0}U_{0}^{\top}W\right\Vert _{\infty}
  \le 3\sum_{i=1}^{4}\left\Vert U_{0}U_{0}^{\top}W_{i}\right\Vert _{\infty}.
\end{align*}
Using the bounds on $\left\Vert U_{0}U_{0}^{\top}W_{i}\right\Vert _{\infty}$
derived above, we obtain that $\left\Vert P_{T}W\right\Vert _{\infty}\le12\cdot\frac{1}{24}\epsilon\lambda\min\left\{ \ca,\cac\right\} =\frac{1}{2}\epsilon\lambda\min\left\{ \ca,\cac\right\} $.

\textbf{(3)} The two equalities in condition (c) in Proposition~\ref{lem:opt_condition} hold by the definition of $ W $. The two inequalities in condition (c) follow from simple algebra as follows. Because $1-q\ge1-t$
and $p\le4t$, we have $\frac{1-q}{p}\ge\frac{1}{4}\frac{1-t}{t}$.
It follows from the conditions in Theorem~\ref{thm:main} that
\begin{equation}\label{eq:bound_p_q}
\begin{aligned}
\frac{p-q}{4}
\ge & c\sqrt{p(1\!-\!q)}\max\left\{ \frac{\sqrt{n}}{K},\sqrt{\frac{\log^{4}n}{K}}\right\}\\ 
\ge & 8p(1\!-\!t)\cdot\frac{48}{\sqrt{t(1\!-\!t)}}\max\left\{ \frac{\sqrt{n}}{K},\sqrt{\frac{\log^{4}n}{K}}\right\} \\
=&8p(1\!-\!t)\epsilon.
\end{aligned}
\end{equation}
We thus have
\[
p-t\ge p-\left(\frac{3}{4}p+\frac{1}{4}q\right)=\frac{p-q}{4}\ge8p(1-t)\epsilon.
\]
One verifies that this implies $(1+\epsilon)\sqrt{\frac{t}{1-t}}\frac{1-p}{p}\le(1-2\epsilon)\sqrt{\frac{1-t}{t}}$.
Plugging in the values of $\ca$
and $\cac$ in~(\ref{eq:weights}) yields 
$$
(1+\epsilon)\frac{\cac(1-p)}{p}  \le  (1-2\epsilon)\ca,
$$ 
Hence, for each $ (i,j)\in R\cap \mca $, we have
\begin{align}
   (U_{0}U_{0}^{\top}+W)_{ij} 
 &=   \frac{1}{p}(U_{0}U_{0}^{\top})_{ij}+(1+\epsilon)\lambda \cac\frac{1-p}{p}\nonumber\\ 
 &\le \frac{1}{p}(U_{0}U_{0}^{\top})_{ij} + (1-2\epsilon)\ca.\label{eq:UU+W}
\end{align}
We also have
\begin{align}\label{eq:UU}
\frac{1}{p}(U_0U_0^\top)_{ij} 
\le \frac{1}{pK} 
\overset{(i)}{\le} \frac{48}{K}\sqrt{\frac{n}{t(1-t)}}\cdot\frac{1}{48}\sqrt{\frac{1-t}{tn}} 
\le \epsilon \cdot \lambda\ca,
\end{align}
where $ (i) $ follows from $p\ge t$. Combining~\eqref{eq:UU+W} and~\eqref{eq:UU} proves the first inequality in the condition (c).

Similarly, we have 
\[
t-q\ge\left(\frac{p}{4}+\frac{3q}{4}\right)-q=\frac{p-q}{4}
\overset{(ii)}{\ge}8p(1-t)\epsilon
\overset{(iii)}{\ge}2t(1-q)\epsilon,
\]
where $ (ii) $ follows from~\eqref{eq:bound_p_q} and $ (iii) $ follows from $p\ge t$ and $1-t\ge1-\frac{3}{4}p-\frac{1}{4}q\ge\frac{1}{4}(1-q)$.
This implies $(1+\epsilon)\sqrt{\frac{1-t}{t}}\frac{q}{1-q}\le(1-\epsilon)\sqrt{\frac{t}{1-t}}$. Therefore, for each $ (i,j)\in R^c\cap \mca^c $, we have
$$
w_{ij} = - (1+\epsilon)\frac{\ca q}{1-q} \ge -(1-\epsilon)\cac,
$$
proving the second inequality in condition (c). This completes the proof of Proposition~\ref{prop:WisGood}.

\section{Proof of Theorem~\ref{thm:converse}}
We use a standard information theoretic argument via Fano's inequality.
For simplicity we assume $n_{1}/K$ and $n_{2}/K$ are both integers,
and we use $c_{1},c_{2}\ldots$ to denote positive absolute constants.
Let $\mathcal{F}$ be the set of all possible ways of assigning $n$
nodes into $n_{1}/K$ clusters of equal size $K$. When $K=\Theta(n_{1})=\Theta(n_{2})$,
the cardinality of $\mathcal{F}$ can be bounded as
\[
M:=\left|\mathcal{F}\right|=\frac{1}{\left(n_{1}/K\right)!}{n \choose K}{n\!-\!K \choose K}\cdots{n_{1}\!+\!K \choose K}\ge c_{2}\cdot c_{1}^{\frac{1}{2}n}
\]
for some $c_{1}>1$ and $c_{2}>0$.

Suppose the true cluster matrix $Y^*$ is obtained uniformly at random from
$\mathcal{F}$, and the graph $A$ is generated from $Y^*$ according
to \gpp with uniform edge probabilities. We use $\mathbb{P}_{A|Y^*}$
to denote the distribution of $A$ given $Y^*$. Let $\hat{Y}$ be
any measurable function of $A$. The Fano's inequality~\cite{cover2012information} gives
\begin{align*}
\sup_{Y^*\in\mathcal{F}}\mathbb{P}\left[\hat{Y}\neq Y^*|Y^*\right]
&\ge1-\frac{I\left(A;Y^*\right)+\log2}{\log M}\\
&\ge1-\frac{I\left(A;Y^*\right)+\log2}{c_{3}n}
\end{align*}
for $n$ is sufficiently large, where $ I(A;Y^*) $ is the mutual information between $ A $ and $ Y^* $. We now bound~$I(A;Y^*$). Let $ H(\cdot) $ denote the Shannon entropy and $ H(\cdot\vert Y^*) $ the Shannon entropy conditioned on $ Y^* $. Observe that 
\begin{align*}
I(A;Y^{*}) 
  = & H(A)-H(A|Y^*)
  \le  \sum_{(i,j):i>j}H(a_{ij})-H(A|Y^*)\\
  = & {n \choose 2}H(a_{12})\!-\!{n \choose 2}H(a_{12}|Y^*)
  =  {n \choose 2}I(a_{12};Y^*),
\end{align*}
where in the second equality we have used the symmetry under the uniform
distribution of $Y^*$ and the conditional independence between $a_{ij}'s$.
By definition of the mutual information, we have
\[
I(a_{12};Y^*)=I(a_{12};y^*_{12})=\mathbb{E}_{y^*_{12}} \left[D\left(\mathbb{P}(a_{12}|y^*_{12})\Vert\mathbb{P}(a_{12})\right)\right].
\]
We can directly compute the divergence on the last RHS. Let $\alpha:=\mathbb{P}(y^*_{12}=1)=\frac{(K-1)n_{1}}{n^{2}}$
and $\gamma:=\mathbb{P}(a_{11}=1)=\alpha p+(1-\alpha)q$.
It follows that
\begin{align*}
 & \mathbb{E}_{y^*_{12}}\left[D\left(\mathbb{P}(a_{12}|y_{12})\Vert\mathbb{P}(a_{12})\right)\right]\\
 = & \alpha p\log\frac{p}{\gamma}+\alpha(1-p)\log\frac{(1-p)}{(1-\gamma)}+(1-\alpha)q\log\frac{q}{\gamma}\\
 &+(1-\alpha)(1-q)\log\frac{(1-q)}{(1-\gamma)}\\
 \le & \alpha p\!\left(\!\frac{p}{\gamma}\!-\!1\!\right)+\alpha(1\!-\!p)\!\left(\!\frac{1\!-\!p}{1\!-\!\gamma}\!-\!1\!\right)+(1\!-\!\alpha)q\!\left(\!\frac{q}{\gamma}\!-\!1\!\right)\!\\
 &+\!(1\!-\!\alpha)(1\!-\!q)\!\left(\!\frac{1\!-\!q}{1\!-\!\gamma}\!-\!1\!\right)\\
 = & \frac{\alpha(1-\alpha)(p-q)^{2}}{\gamma(1-\gamma)}\le c_{4}\frac{(p-q)^{2}}{p(1-q)},
\end{align*}
where in the last inequality we use $\gamma\ge\alpha p$, $1-\gamma\ge(1-\alpha)(1-q$)
and $\alpha,1-\alpha=\Theta(1)$. Combining pieces, we obtain
$$
\sup_{Y\in\mathcal{F}}\mathbb{P}\left[\hat{Y}\neq Y|Y\right] 
 \ge  1-\frac{c_{5}\frac{(p-q)^{2}n^{2}}{p(1-q)}+\log2}{c_{3}n}.
$$
For the last R.H.S. to be less than $\frac{1}{4}$, we need $\frac{(p-q)^{2}}{p(1-q)}\ge c_{6}\frac{1}{n}$.
This completes the proof of the theorem.

\section{Proof of Theorem~\ref{thm:est_pq}}
\label{sec:proof_est_pq}

Suppose the eigenvalues of the matrix $\mathbb{E}[A]$ are  $\lambda_{1}\ge \lambda_2\ge\cdots\ge\lambda_n$, whose values are computed in Section~\ref{sec:est_p}. Observe that the matrix $A-\mathbb{E}A$ is
a random symmetric matrix with independent zero-mean entries, each of which
is bounded in absolute value by $1$ and has variance bounded by $\max\{p(1-p), q(1-q)\}\le p(1-q)$.
Under the condition of Theorem~\ref{thm:est_pq}, we may apply Lemma~\ref{lem:random_matrix}
to obtain $\left\Vert A-\mathbb{E}A\right\Vert \le4\sqrt{p(1-q)n}$ w.h.p.
It then follows from Weyl's inequality~\cite{bhatia1987perturbation}
that w.h.p. 
\begin{equation}\label{eq:perturbation}
\max_{i}\left\{ \left|\hat{\lambda}_{i}-\lambda_{i}\right|\right\} \le\left\Vert A-\mathbb{E}A\right\Vert \le 4\sqrt{p(1-q)n}.
\end{equation}
In the sequel, we assume we are on the event that~\eqref{eq:perturbation} holds.

\paragraph{Estimation of $r$ }

Recall that $\lambda_{1}=K(p-q)+nq+(1-p)$, $\lambda_{2},\ldots,\lambda_{r}=K(p-q)+(1-p)$,
and $\lambda_{r+1},\ldots,\lambda_{n}=1-p$. The inequality~\eqref{eq:perturbation} implies that for some universal constant $ c_1 $:
\begin{itemize}
\item $\hat{\lambda}_{1}-\hat{\lambda}_{2}\le\lambda_{1}-\lambda_{2}+\left|\hat{\lambda}_{1}-\lambda_{1}\right|+\left|\hat{\lambda}_{2}-\lambda_{2}\right|\le nq+c_1\sqrt{p(1-q)n}$; 
\item similarly, $\hat{\lambda}_{i}-\hat{\lambda}_{i+1}\le c_1\sqrt{p(1-q)n}$
for $i=2,\ldots r-1$ and $i\ge r+1$; 
\item $\hat{\lambda}_{r}-\hat{\lambda}_{r+1}\ge\lambda_{r}-\lambda_{r+1}-\left|\hat{\lambda}_{r}-\lambda_{r}\right|-\left|\hat{\lambda}_{r+1}-\lambda_{r+1}\right|\ge K(p-q)-c_1\sqrt{p(1-q)n}$. 
\end{itemize}
Under the condition~\eqref{eq:main_cond}, we have $K(p-q)\ge c_2\sqrt{p(1-q)n}$
for some constant $c_2$. This implies $\hat{\lambda}_{r}-\hat{\lambda}_{r+1}>\frac{K(p-q)}{2}>\hat{\lambda}_{i}-\hat{\lambda}_{i+1}$
for all $i>1$ and $i\neq r$. This guarantees $\hat{r}=r$ and thus
$\hat{K}=K$.

\paragraph{Estimation of $p$ and $q$}

By~\eqref{eq:perturbation} and the triangle inequality, the estimation error of $\hat{q}$ satisfies 
\begin{align*}
|\hat{q}-q| 
  &=\left|\frac{\hat{\lambda}_{1}-\lambda_{1}}{n}-\frac{\hat{\lambda}_{2}-\lambda_{2}}{n}\right|
  \le c_3\frac{\sqrt{p(1-q)n}}{K}.
\end{align*}
Similarly, we have 
\begin{align*}
|\hat{p}-p| 
  = & \left|\frac{\hat{K}\hat{\lambda}_{1}+(n-\hat{K})\hat{\lambda}_{2}-n}{n(\hat{K}-1)}-\frac{K\lambda_{1}+(n-K)\lambda_{2}-n}{n(K-1)}\right|\\
  = & \left|\frac{K(\hat{\lambda}_{1}-\lambda_{1})+(n-K)(\hat{\lambda}_{2}-\lambda_{2})}{n(K-1)}\right|\\
  \le& c_3\frac{\sqrt{p(1-q)n}}{K}.
\end{align*}

\paragraph{Choosing $t$}

Using the above bounds on $\hat{p}$ and $\hat{q}$, we obtain 
\begin{align*}
t  = & \frac{p+q}{2}+\frac{\hat{p}-p+\hat{q}-q}{2}
  \le  \frac{p+q}{2}+c_4\frac{\sqrt{p(1-q)n}}{K}\\
  \le & \frac{p+q}{2}+\frac{p-q}{4}=\frac{3}{4}p+\frac{1}{4}q,
\end{align*}
where in the last inequality we use $\frac{p-q}{4}\ge c_4\frac{\sqrt{p(1-q)}}{K}$, satisfied under the condition~\eqref{eq:main_cond}. This proves one side of the interval for $ t $. The other side is proved in a similar way.

\section{Conclusion}\label{sec:conclusion}
This work is motivated by clustering large-scale networks such as modern online social networks, where the graphs are often highly noisy and have heterogeneous and non-random structures. We considered a natural and versatile model, namely the semi-random Generalized Stochastic Blockmodel, for clustered random graphs. This model recovers many classical generative models for graph clustering. We presented a convex optimization formulation, essentially a convexification of the maximum likelihood estimator. Our theoretic analysis shows that this method is guaranteed to recover the correct clusters under a wide range of parameters of the problem. In fact, our method order-wise outperforms existing methods in this setting, in the sense that it succeeds under less restrictive conditions. Experiment results also validate the effectiveness of the proposed method.

Possible directions for future work include faster algorithm implementations, developing effective post-processing/rounding schemes when the obtained solution is not an exact cluster matrix, and extension to online clustering settings (e.g., via incremental stochastic optimization~\cite{feng2013online}). It is also interesting to extend the algorithms and analysis to more general settings beyond the models in Definitions~\ref{def:GSBM} and~\ref{def:SR_GSBM}, for example, when the in-cluster and cross-cluster densities are not bounded uniformly and the clusters have overlaps.

\appendix

\section{Technical Lemmas}
In this section, we record two technical lemmas that are needed in the proofs of our theoretical results. The first lemma is a standard bound on the spectral norm of a random symmetric matrix.
\begin{lem}
\label{lem:random_matrix} Suppose $ Y $ is a symmetric $ n\times n $ matrix, where $Y_{ij}$, $1\le i,j\le n$ are independent random variables, each of which has mean $0$ and variance at most $\sigma^{2}$ and is bounded in absolute value by $B$. If $\sigma\ge c_1\frac{B\log^{2}n}{\sqrt{n}}$ for some absolute constant $ c_1>0 $, 
then with probability at least $1-n^{-10}$, 
\[
\left\Vert Y\right\Vert\le4\sigma\sqrt{n}.
\]
\end{lem}
\begin{proof}
Except for $ Y $ being symmetric, the proof is the same as that of Theorem~3.1 in~\cite{achlioptas2007fast}.
\end{proof}
The second lemma is a restatement of the standard Bernstein inequality for the sum of independent random variables.
\begin{lem}
\label{lem:subgaussian_concentr} $Let$ $Y_{1},\ldots,Y_{N}$ be independent random variables,
each of which  is bounded
in absolute value by $B$ a.s. and has variance bounded by $\sigma^{2}$. For any constant $ c_1>0$, there exists a constant $c_{0}>0$ independent of $\sigma$,
$B$, $N$ and $n$ such that for any  $ n\ge 1 $, if $\sigma\ge B\sqrt{\frac{\log n}{N}}$,
then we have 
\[
\left|\sum_{i=1}^{N}Y_{i}-\mathbb{E}\left[\sum_{i=1}^{N}Y_{i}\right]\right|\le c_{0}\sigma\sqrt{N\log n}
\]
with probability at least $1-2n^{-c_{1}}$.
\end{lem}

\bibliographystyle{IEEEtran}
\bibliography{sparse}


\end{document}